\documentclass[
    sigconf = true,     
    screen = true,      
]{acmart}

\newboolean{arxiv} 
\setboolean{arxiv}{true} 

\usepackage{xspace}
\usepackage{algpseudocode}
\usepackage[ruled,vlined,linesnumbered]{algorithm2e}
\usepackage{wrapfig}

\allowdisplaybreaks[4]
\newcommand{\oea}{\mbox{$(1 + 1)$~EA}\xspace}

\newcommand{\oplea}{\mbox{$(1+\lambda)$~EA}\xspace}

\newcommand{\oclea}{\mbox{$(1,\lambda)$~EA}\xspace}

\newcommand{\MAHH}{MAHH\xspace}

\newcommand{\onemax}{\textsc{OneMax}\xspace}
\newcommand{\LO}{\textsc{Leading\-Ones}\xspace}
\newcommand{\leadingones}{\LO}

\newcommand{\cliff}{\textsc{Cliff}\xspace}

\newcommand{\jump}{\textsc{Jump}\xspace}

\newcommand{\eps}{\varepsilon}

\let\originalleft\left
\let\originalright\right
\renewcommand{\left}{\mathopen{}\mathclose\bgroup\originalleft}
\renewcommand{\right}{\aftergroup\egroup\originalright}

\newcommand{\eqnref}[1]{(\ref{#1})}

\AtBeginDocument{%
  \providecommand\BibTeX{{%
    \normalfont B\kern-0.5em{\scshape i\kern-0.25em b}\kern-0.8em\TeX}}}

\setcopyright{acmcopyright}
\copyrightyear{2023}
\acmYear{2023}
\setcopyright{acmlicensed}\acmConference[GECCO '23]{Genetic and
Evolutionary Computation Conference}{July 15--19, 2023}{Lisbon, Portugal}
\acmBooktitle{Genetic and Evolutionary Computation Conference (GECCO '23),
July 15--19, 2023, Lisbon, Portugal}
\acmPrice{15.00}
\acmDOI{10.1145/3583131.3590509}
\acmISBN{979-8-4007-0119-1/23/07}

\begin{document}

\title{How the Move Acceptance Hyper-Heuristic Copes With Local Optima: Drastic Differences Between Jumps and Cliffs}\ifthenelse{\boolean{arxiv}}{\titlenote{Author-generated version.}}{}

\author{Benjamin Doerr}
\affiliation{%
  \institution{Laboratoire d’Informatique (LIX),\\
École Polytechnique, CNRS, \\
Institut Polytechnique de Paris,}
  \city{Palaiseau}
  \country{France}
}

\author{Arthur Dremaux}
\affiliation{%
  \institution{
École Polytechnique, \\
Institut Polytechnique de Paris,}
  \city{Palaiseau}
  \country{France}
}

\author{Johannes Lutzeyer}
\affiliation{%
  \institution{Laboratoire d’Informatique (LIX),\\
École Polytechnique, CNRS, \\
Institut Polytechnique de Paris,}
  \city{Palaiseau}
  \country{France}
}

\author{Aur\'elien Stumpf}
\affiliation{%
  \institution{
  École Polytechnique, \\
Institut Polytechnique de Paris,}
  \city{Palaiseau}
  \country{France}
}

\renewcommand{\shortauthors}{Benjamin Doerr, Arthur Dremaux, Johannes Lutzeyer, and Aur\'elien Stumpf}

\begin{abstract}
In recent work, Lissovoi, Oliveto, and Warwicker (Artificial Intelligence (2023)) proved that the Move Acceptance Hyper-Heuristic (\MAHH) leaves the local optimum of the multimodal cliff benchmark  with remarkable efficiency. With its $O(n^3)$ runtime, for almost all cliff widths $d,$ the \MAHH massively outperforms the $\Theta(n^d)$ runtime of simple elitist evolutionary algorithms (EAs). For the most prominent multimodal benchmark, the jump functions, the given runtime estimates of $O(n^{2m} m^{-\Theta(m)})$ and $\Omega(2^{\Omega(m)})$, for gap size $m \ge 2$, are far apart and the real performance of \MAHH is still an open question.
  
In this work, we resolve this question. We prove that for any choice of the \MAHH selection parameter~$p$, the expected runtime of the \MAHH on a jump function with gap size $m = o(n^{1/2})$ is at least $\Omega(n^{2m-1} / (2m-1)!)$. This renders the \MAHH much slower than simple elitist evolutionary algorithms with their typical $O(n^m)$ runtime. 
  
We also show that the \MAHH with the global bit-wise mutation operator instead of the local one-bit operator optimizes jump functions in time $O(\min\{m n^m,\frac{n^{2m-1}}{m!\Omega(m)^{m-2}}\})$, essentially the minimum of the optimization times of the \oea and the \MAHH. This suggests that combining several ways to cope with local optima can be a fruitful approach. 
\end{abstract}


\ccsdesc[300]{Theory of computation~Theory of randomized search heuristics}

\ccsdesc[500]{Mathematics of computing~Evolutionary algorithms}

\keywords{Keywords: Hyper-heuristics, non-elitism, mutation, cliff benchmark.}

{\sloppy
		
\maketitle

\ifthenelse{\boolean{arxiv}}{\pagestyle{plain}}{}

\section{Introduction}

Compared to the intensive analysis of evolutionary algorithms (EAs) via mathematical means over the last 30 years~\cite{NeumannW10,AugerD11,Jansen13,DoerrN20}, the rigorous analysis of hyper-heuristics is much less developed and has so far mostly focused on unimodal problems (however, with highly interesting results, see Section~\ref{sec:prev} for more detail).

In the first and so far only work on how hyper-heuristics solve multimodal problems, that is, how they cope with the presence of true local optima, Lissovoi, Oliveto, and Warwicker have shown the remarkable result that the simple move-acceptance hyper-heuristic (\MAHH), which randomly mixes elitist selection with accepting any new solution, optimizes all \cliff functions in time $O(n^3)$ (and depending on the cliff width~$d$ even better). The \MAHH thus struggles much less with the local optimum of this benchmark than elitist EAs (typically having a runtime of $\Theta(n^d)$~\cite{PaixaoHST17}), the \oclea having a runtime of roughly $O(n^{3.98})$ (shown for $d=n/3$ only, this result requires a very careful choice of the population size)~\cite{FajardoS21foga}, and the Metropolis algorithm, which has a runtime of $\Omega(n^{d-0.5} / (\log n)^{d-1.5})$ for constant $d$ and a super-polynomial runtime for super-constant~$d$.

The surprisingly good performance of the \MAHH on \cliff functions raises the question of whether the convincing performance on \cliff of the \MAHH generalizes to other functions or whether it a particularity of the \cliff benchmark. To answer this question, we study the performance of the \MAHH on the multimodal benchmark most prominent in the mathematical runtime analysis of randomized search heuristics, the \jump benchmark. For this problem, with jump size $m\ge 2$, only the loose bounds 
$O(n^{2m} m^{-\Theta(m)})$ and $\Omega(2^{\Omega(m)})$
were shown in~\cite{LissovoiOW23}. These bounds allow no conclusive comparison with simple evolutionary algorithms, which typically have a  $\Theta(n^m)$ runtime. 

In this work, we prove a general non-asymptotic lower bound for the runtime of the \MAHH on \jump functions, valid for all values of the problem size~$n$, the jump size~$m \ge 2$, and the mixing parameter~$p$ of the hyper-heuristic. As the most interesting special case of this bound, we derive that for $m = o(n^{1/2})$ and all values of $p$, this runtime is at least $\Omega(\frac{n^{2m-1}}{(2m-1)!})$. This is significantly larger than the $\Theta(n^m)$ runtime of many evolutionary algorithms. 

Our lower bound is relatively tight. We prove that the \MAHH with $p = \frac mn$ has a runtime of $O(\frac{n^{2m-1}}{m! m^{m-2}})$. In particular, for constant jump size~$m$, which is the only regime in which polynomial runtimes exist, the runtime of the \MAHH with optimal parameter choice is $\Theta(n^{2m-1})$. Therefore, our bounds allow us to answer the question on the generalization of the good performance of the \MAHH: The \MAHH performs comparatively poorly on \jump functions, i.e., the surprisingly good performance on \cliff does not extend to \jump functions. 

Noting the significant performance gap between, e.g., the \oea and the \MAHH, we propose to use the \MAHH with bit-wise mutation, the variation operator of the \oea, instead of one-bit flips. This global mutation operators renders the analysis of the \MAHH significantly more complex and in particular forbids the use of our lower bound arguments. With a suitable potential function in the drift argument, we manage to adapt our upper bound proof and show an upper bound of $O(\min\{m n^m,\frac{n^{2m-1}}{m!\Omega(m)^{m-2}}\})$. This shows that, in principle, combining the global variation operator popular in evolutionary computation with a local search hyper-heuristic can be an interesting approach. We note that this was suggested for the Metropolis algorithm in~\cite{DoerrERW23}, but analyzed only via experiments on \cliff functions with problem size $n=100$ and cliff height $d=3$. In particular, no proven performance guarantees were given in~\cite{DoerrERW23}. 

We believe that this is a promising direction for future research, but we also note that we were not able to prove an $O(n^3)$ upper bound for this new algorithm on \cliff functions, so we cannot rule out that the strong performance of the classic \MAHH on \cliff functions is lost when using bit-wise mutation. More research in this direction to answer such questions is needed. At the moment, the main missing piece towards more progress is a method to prove lower bounds when less restricted mutation operators than one-bit flips are used.

\section{Previous Works}\label{sec:prev}

We now describe the main previous works. This being a mathematical runtime analysis, we concentrate on such works. We refer to~\cite{BurkeGHKOOQ13} for a general introduction to hyper-heuristics. We refer to~\cite{NeumannW10,AugerD11,Jansen13,DoerrN20} for introductions to mathematical runtime analyses of randomized search heuristics.

\subsection{Hyper-Heuristics}

The mathematical runtime analysis of randomized search heuristics, mostly concerned with the analysis of evolutionary algorithms, is an established field for more than 30 years now~\cite{NeumannW10,AugerD11,Jansen13,DoerrN20}. The mathematical runtime analysis of hyper-heuristics was started much more recently by Lehre and \"Oczan~\cite{LehreO13} in work on random mixing of mutation operators and random mixing of acceptance operators. We refer to the survey~\cite[Section~8]{DoerrD20bookchapter} for a detailed discussion of random mixing of mutation operators and, in particular, a comparison to very similar earlier works that were not called hyper-heuristics at their time such as~\cite{GielW03,NeumannW07}. We discuss the results on acceptance operators further below.

In~\cite{AlanaziL14}, for the first time, more complex selection hyper-heuristics for choosing suitable mutation operators were studied, however, no superiority over random mixing could be shown. 
That this is not the fault of the authors, but rather a misconfiguration of these hyper-heuristics, was shown in the seminal work of Lissovoi, Oliveto, and Warwicker~\cite{LissovoiOW20ecj}. In this work, the authors first show that the three more complex hyper-heuristics optimize the \leadingones benchmark in exactly (ignoring lower order terms) the same runtime as the simple random mixing heuristics. 
They observe that the true reason for this behavior is the short-sightedness of these heuristics: As soon as no improvement is found, they switch to a different low-level heuristic. If instead longer learning periods are used, that is, a change of the currently used low-level heuristic is made only after a certain number of successive failures, then the more complex hyper-heuristics perform very well and attain the best possible performance that can be obtained from the given low-level heuristics. 
The length of this learning period is crucial for the success of the hyper-heuristic, but as shown in~\cite{DoerrLOW18}, it can be determined in a self-adjusting manner. 
These works were extended to other benchmark problems in~\cite{LissovoiOW20aaai}.

Besides some artificial problems constructed to demonstrate particular effects, all works above regard unimodal benchmarks. The first and so far only work to analyze the performance of hyper-heuristics on multi-modal problems is~\cite{LissovoiOW23}. This work takes up the random mixing of selection operators idea of~\cite{LehreO13}. There, a random mixing between an all-moves ($\text{ALLMOVES}$) operator, accepting any new solution, and an only-improving ($\text{ONLYIMPROVING}$) operator, accepting only strict improvements, was proposed and it was shown that the (unimodal) royal-road problem can only be solved by mixing these two operators. As noted in~\cite{LissovoiOW23}, this result heavily relies on the fact that the $\text{ONLYIMPROVING}$ operator was used instead of the, in evolutionary computation more common, improving-and-equal ($\text{IMPROVINGANDEQUAL}$) operator, which accepts strict improvements and equally good solutions. 

The results in~\cite{LissovoiOW23} do not depend on this particularity. For this reason, let us refer to the hyper-heuristic which randomly mixes between the $\text{ALLMOVES}$ (with probability~$p$, the only parameter of the algorithm) and either of the $\text{ONLYIMPROVING}$ and $\text{IMPROVINGANDEQUAL}$ operator, as the  \emph{move-acceptance hyper-heuristic (\MAHH)}. In other words, the \MAHH heuristic starts with a random solution and then, in each iteration, moves to a random neighbor of the current solution and, with probability~$p$  always accepts this move  and with probability $1-p$ only accepts this move if the new solution is at least as good or strictly better (for the problems discussed in the remainder, both versions will behave identical). See Section~\ref{sec:preliminaries} for a more detailed description of the algorithm. 

The striking result of~\cite{LissovoiOW23} is that the simple \MAHH with mixing parameter $p = \frac{1}{(1+\eps)n}$, $\eps > 0$ a constant, optimizes the \cliff benchmark  with cliff width~$d$ in an expected number of $O(\frac{n^3}{d^2} + n \log n)$ iterations. This is remarkably efficient compared to the runtimes of other algorithms on this benchmark, see Section \ref{sec:runtime_literature}. This result raises the question if \MAHH has the general strength of leaving local optima or whether this behavior is particular to the \cliff problem, which with its particular structure (a single acceptance of an inferior solution suffices to leave the local optimum) appears to be the perfect benchmark for the \MAHH. 

To answer this question, the authors of~\cite{LissovoiOW23} regard also other multimodal problems, however, with non-conclusive results. 

For the most prominent multimodal benchmark $\jump_m$, $m \ge 2$, in~\cite{LissovoiOW23} the bounds $\Omega(2^{\Omega(m)} + n \log n)$ and $O(\frac{(1+\eps)^{m-1} n^{2m-1}}{m^2 m!})$ were shown for the parameter choice $p=\frac{1}{(1+\eps)n}$. We note that the conference version~\cite{LissovoiOW19} states, without proof, a stronger upper bound of $O(\frac{n^{2m-1}}{m})$, but since the journal version only proves a weaker bound, we assume that 
only the bound in the journal version is valid. In any case, these bounds are significantly distant to the known runtime $\Theta(n^m)$ of simple evolutionary algorithms on the jump benchmark, and thus do not allow a conclusive comparison of these algorithms. 

\subsection{Runtimes Analyses on Cliff and Jump Functions}\label{sec:runtime_literature}

We now briefly collect the main runtime results for the \cliff and \jump benchmarks. \cliff functions (with fixed cliff width $d =  n/3$) were introduced in~\cite{JagerskupperS07} to show that the \oclea can profit from its non-elitism. This result was made precise in~\cite{FajardoS21foga}. For the best choice of the population size, a runtime of approximately $O(n^{3.98})$ was shown, however, the result also indicates that small deviations from this optimal parameter choice lead to significant performance losses. 

Comparably simple elitist evolutionary algorithms can leave the local optimum of a general \cliff function, and similarly \jump function, with cliff width $d$ only by flipping the right $d$ bits, hence they have a runtime of $\Theta(n^d)$ when using bit-wise mutation with mutation rate $\frac 1n$ (this follows essentially from the result on \jump functions in the classic work~\cite{DrosteJW02}, but was shown separately for \cliff functions in~\cite{PaixaoHST17}). When using the heavy-tailed mutation operator proposed in~\cite{DoerrLMN17}, again as for \jump functions, the runtime reduces to $n^d d^{-\Theta(d)}$. A combination of mathematical and experimental evidences suggests that the compact genetic algorithm has an exponential runtime on the \cliff function with $d=n/3$~\cite{NeumannSW22}.

The Metropolis algorithm profits at most a little from its ability to leave local optima when optimizing \cliff functions. For constant~$d$, a lower bound of $\Omega(n^{d-0.5} (\log n)^{-d+1.5})$ was shown in~\cite{LissovoiOW23}, for super-constant $d$ it was shown that the runtime is super-polynomial. For a recent tightening and extension of this result, we refer to~\cite{DoerrERW23}.

We note that some artificial immune systems employing an aging operator were shown to optimize \cliff functions in time $O(n \log n)$~\cite{CorusOY20}. Since artificial immune systems are even less understood than hyper-heuristics, it is hard to estimate the general meaning of this result. Hence in summary, we agree with the authors of~\cite{LissovoiOW23} that the performance of the \MAHH on \cliff functions is remarkably good. 

The \jump functions benchmark is by far the most studied multimodal benchmark in the theory of randomized search heuristics. It was proposed already in~\cite{DrosteJW02}, where the runtime of the \oea on this benchmark was shown to be $\Theta(n^m)$ for all values of $m$. Since then is has been intensively studied and given rise to many important results, e.g., it is one of the few examples where crossover was proven to give significant performance gains~\cite{JansenW01,KotzingST11,DangFKKLOSS18,AntipovDK22}, estimation-of-distribution algorithms and ant-colony optimizers were shown to significantly outperform classic evolutionary algorithms on \jump functions~\cite{HasenohrlS18,Doerr21cgajump,BenbakiBD21}, and it led to the development of fast mutation~\cite{DoerrLMN17} and a powerful stagnation-detection mechanism~\cite{RajabiW22}. Several variations of jump functions have been proposed and analyzed~\cite{Jansen15, BamburyBD21, RajabiW21gecco, DoerrZ21aaai, FriedrichKKR22, DoerrQ23tec, DoerrQ23crossover, DoerrQ23LB, Witt23, BianZLQ23}.

In the context of our work, it is important to note that typical elitist mutation-based algorithms optimize \jump functions in time $\Theta(n^m)$. A speed-up by a factor of $\Omega(m^{\Omega(m)})$ can be obtained from fast mutation~\cite{DoerrLMN17} and various forms of stagnation detection~\cite{RajabiW22,RajabiW23,RajabiW21gecco,DoerrR23} (where stagnation detection usually gives runtimes by a factor of around $\sqrt m$ smaller than fast mutation). With crossover, the $(\mu+1)$ GA without additional modifications reaches runtimes of $\tilde O(n^{m-1})$~\cite{DangFKKLOSS18,DoerrEJK23arxiv}. With suitable diversity mechanisms or other additional techniques, runtimes up to $O(n)$ were obtained~\cite{DangFKKLOSS16,FriedrichKKNNS16,WhitleyVHM18,RoweA19,AntipovD20ppsn,AntipovBD21gecco,AntipovDK22}, but the lower the runtimes become, the more these algorithms appear custom-tailored to jump functions, see, e.g.,~\cite{Witt23}. The extreme end is marked by an $O(\frac{n}{\log n})$ time algorithm~\cite{BuzdalovDK16} designed to witness the black-box complexity of jump functions.  

Non-elitism could not be used effectively on \jump functions so far. The \oclea for essentially all reasonable parameter choices was proven to have at least the runtime of the \oplea~\cite{Doerr22}. In~\cite[Theorem~14]{LissovoiOW23}, the Metropolis algorithm for any value of the acceptance parameter was shown to have a runtime of at least $2^{\Omega(n)}$ with at least constant probability.\footnote{The result~\cite[Theorem~14]{LissovoiOW23} says ``with probability $1-2^{-\Omega(n)}$'', but this seems to overlook that the proof from the second sentence on only regards the case that the random initial solution has at least $n/2$ zeroes.}

\section{Preliminaries}\label{sec:preliminaries}

We now formally define the \MAHH algorithm, our considered standard benchmark functions as well as some mathematical tools used in our runtime analysis. 

\subsection{Algorithms}\label{sec:algorithms}

We will analyse the runtime of the \MAHH algorithm applied to the problem of reaching the optima of a benchmark function defined on the space of $n$-dimensional bit vectors.

In each iteration of the algorithm, one bit of the current vector $x$ is chosen at random and flipped to create a mutation $x^{\prime}$. This mutation is accepted with probability $p$ (ALLMOVES operator), and accepted only if the value of the benchmark function increases with probability $1-p$ (ONLYIMPROVING operator). We provide a detailed overview of the Move Acceptance Hyper-Heuristic in Algorithm~\ref{alg:mahh}.

\begin{algorithm}[t]
\caption{Move Acceptance Hyper-Heuristic (\MAHH).}\label{alg:mahh}
\KwData{Choose $x \in \{0,1 \}^n$ uniformly at random}
\While{\text{termination criterion not satisfied}}{
  $x^{\prime} \gets\text{FLIP-RANDOM-BIT}(x)$\;
  $\text{ACC} \gets 
  \begin{cases}
      \text{ALLMOVES}, & \text{with probability } p; \\
      \text{ONLYIMPROVING}, & \text{otherwise};\\
  \end{cases}$\\
  \If{$\text{ACC}(x,x^{\prime})$}{
    $x \gets x^{\prime}$\;
  }
}
\end{algorithm}

We will furthermore compare the runtime of the \MAHH algorithm to the performance of several well-known algorithms. In particular, we will be considering the METROPOLIS Algorithm, in Algorithm \ref{alg:metropolis}, and the (1+1) Evolutionary Algorithm, in Algorithm~\ref{alg:evolalg}.

\begin{algorithm}[t]
\caption{METROPOLIS Algorithm.}\label{alg:metropolis}
\KwData{Choose $x \in \{0,1 \}^n$ uniformly at random}
\While{\text{termination criterion not satisfied}}{
  $x^{\prime} \gets\text{FLIP-RANDOM-BIT}(x)$\;
  $\Delta f \leftarrow{} f(x^{\prime}) - f(x)$\;
  \eIf{$\Delta f \geq 0$}{
    $x \gets x^{\prime}$\;
  }
  {
  $r \gets u \in \left[0,1\right]$ chosen uniformly at random\;
  \If{$r \leq \alpha (n)^{\Delta f}$ }{
    $x \gets x^{\prime}$\;
  }
  }
}
\end{algorithm}

\begin{algorithm}[t]
\caption{(1+1) Evolutionary Algorithm.}\label{alg:evolalg}
\KwData{Choose $x \in \{0,1 \}^n$ uniformly at random}
\While{\text{termination criterion not satisfied}}{
  $x \gets\text{flip each bit of } x \text{ with probability } \frac{1}{n}$\;
  $\Delta f \leftarrow{} f(x^{\prime}) - f(x)$\;
  \If{$\Delta f \geq 0$}{
    $x \gets x^{\prime}$\;
  }
}
\end{algorithm}

\subsection{Benchmark Function Classes}

We will now define the various benchmark functions that we will use to analyse the performances of the algorithm. 

We first define the $\onemax$ function as
$$\onemax(x) = \sum_{i=1}^{n}x_i.$$
The $\onemax$ function has a constant slope leading to a global optimum placed at the $1^n$ bit-string. It is used to evaluate the hillclimbing performance of randomized search heuristics. 

We define the $\cliff_d$ class of functions (for $1 \leq d \leq n/2$) as follows, 
\begin{eqnarray*}
\cliff_d(x) = 
\begin{cases}
      \onemax(x), & \text{if } \lVert x\rVert_1 \leq n-d; \\
      \onemax(x) - d + 1/2, & \text{otherwise.}\\
\end{cases}  
\end{eqnarray*}
The $\cliff_d$ class of functions are examples of functions where evolutionary algorithms only accepting the best proposed moves seen so far (called elitist evolutionary algorithms) will perform poorly. It is therefore used to evaluate the ability of the algorithm to escape local optima. 

In this paper, we will present a runtime analysis of the \MAHH hyperheuristic on the $\jump_m$ class of functions defined as follows,
\begin{eqnarray*}
\jump_m(x) = 
\begin{cases}
      n+m, & \text{if } \lVert x\rVert_1 = n; \\
      m+\onemax(x), & \text{if } \lVert x\rVert_1 \leq n-m; \\
      n-\onemax(x), & \text{otherwise.}\\
\end{cases}  
\end{eqnarray*}
The $\jump_m$ class of functions are examples of functions where the local optimum has a wide basin of attraction, which will make it more difficult for the hyper-heuristics to escape the basin and reach the global optimum.


\subsection{Mathematical Tools}

We now introduce several of the mathematical tools, that we use to prove our lower and upper bound in Sections~\ref{sec:lower_bound}~and~\ref{sec:upper_bound}, respectively.

The expected time to reach a state with $i+1$ one-bits, given a state with $i$ one-bits can be obtained using the following recurrence formula.

\begin{lemma}[\cite{DrosteJW00}] \label{lma1} 
We denote by $T_{i}^{+}$ the expected time to reach a state with $i+1$ one-bits, given a state with $i$ one-bits. We denote by $p_i^{-}$ and $p_i^{+}$ the transition probabilities to  reach states with $i-1$ and $i+1$ one-bits, respectively. Then
$$\mathbb{E}\left[ T_{i}^{+} \right] = \frac{1}{p_i^{+}} + \frac{p_i^{-}}{p_i^{+}} \mathbb{E}\left[ T_ {i-1}^+ \right].$$
\end{lemma}

Lemma \ref{lma1} allows us to link a step of the algorithm to previous steps while taking the local variations of the benchmark function into account. The formula can be expanded to give an explicit formula for the expected time to increase the number of one-bits by one.

\begin{lemma} \label{lma2} 
Following the notation defined in Lemma \ref{lma1},
$$\mathbb{E}[T_{i}^{+}] = \sum_{k=0}^{i} \frac{1}{p_k^{+}} \prod_{\ell=k+1}^{i} \frac{p_\ell^{-}}{p_\ell^{+}}.$$
\end{lemma}

Each term of the sum can be interpreted as the expected time of the path starting from the state with $i$ one-bits and returning to the state with $k$ one-bits ($k < i$) before reaching the state with $i+1$ one-bits. 

The following Multiplicative Drift Theorem provides upper bounds on the expected runtime in the case where the progress can be bounded from below by an expression proportional to the distance from the target (which is zero in the formulation of the theorem below). 

\begin{theorem}[Multiplicative Drift Theorem \cite{DoerrJW12algo}]
\label{thm:mult-drift}
Let $S \subseteq \mathbb{R}$ be a finite set of positive numbers with minimum $s_{min}$. Let $(X_t)_{t\geq0}$ be a sequence of random variables over $S \cup \{0\}$. Let $T$ be the random variable that denotes the first point in time $t \in \mathbb{N}$ for which $X_t = 0$. 
Suppose further that there exists a constant $\delta > 0$ such that 
$$\mathbb{E}\left[X_{t} - X_{t+1} \mid  X_t = s \right] \geq \delta s,$$
holds for all $s \in S$ with $\mathbb{P}\left[X_t = s_0\right] > 0$. 
Then, for all $s_0 \in S$ with $\mathbb{P}\left[X_0 = s_0 \right] > 0$,
$$\mathbb{E}\left[T \mid  X_0 = s_0 \right] \leq \frac{1 + \log(s_0/s_{min})}{\delta}.$$
\end{theorem}

In our analysis in Section~\ref{sec:upper_bound} we furthermore make use of Wald's formula, which we introduce now. 

\begin{theorem}[Wald's Formula \cite{Wald44}]
\label{thm:wald}
If we assume $(X_n)_{n \geq 0}$ to be a sequence of random independent variables following the same law and to be integrable, and $N$ to be a waiting time integrable and adapted to $(X_{n})_{n \geq 0}$, i.e., the events $\{N = n\}$ are determined by $(X_1,...,X_n)$. 
Then, 
$$
\mathbb{E} \left[ \sum_{i=1}^{N} X_i \right] = \mathbb{E}\left[N\right] \mathbb{E}\left[X_1\right].
$$

\end{theorem}

\section{Lower Bound on the Runtime of \MAHH  on $\jump_m$} 
\label{sec:lower_bound}

In this section, we prove a general formula for $\mathbb{E}[T_{n-1}^{+}]$ by using Lemma~\ref{lma2} for $i=n-1$, and use it to compute more precise lower bounds for the runtime of the \MAHH algorithm.

\begin{theorem}[General Formula for the Expected Duration of the Last Step]
\label{thmformula}
We consider the \MAHH algorithm on $\jump_m$.
Let $T_{n-1}^{+}$ be the expected time for the \MAHH algorithm to reach the state with $n$ 1-bits, given a state with $(n-1)$ 1-bits. For any $p > 0,n$ and $m \leq n$, we have 
 \begin{equation} \label{eqn:GeneralFormula}
  \mathbb{E}[T_{n-1}^{+}] = p^{n-2m+1} \sum_{k=0}^{n-m-1} p^{-k} \binom{n}{k} + p^{1-n} \sum_{k = n-m}^{n-1} \binom{n}{k} p^{k}.     
 \end{equation}
\end{theorem}

\begin{proof}
This result is a direct consequence of Lemma~\ref{lma2} for $i = n-1$,
\begin{equation}\label{eqn:ExpectedTimeLemma}
\mathbb{E}[T_{n-1}^{+}] = \sum_{k=0}^{n-1} \frac{1}{p_k^{+}} \prod_{\ell=k+1}^{n-1} \frac{p_\ell^{-}}{p_\ell^{+}}.
\end{equation}
We now need to compute the $ \frac{p_\ell^{-}}{p_\ell^{+}}$ terms. A straightforward calculation counting the number of bits that we can flip, and adding a factor $p$ in cases where the $\jump_m$ score reduces, gives us

\begin{equation}\label{eqn:PossibleFractionValues}
\frac{p_k^{-}}{p_k^{+}} = \left\{
    \begin{array}{ll}
        \frac{k}{n-k} p, & \mbox{if } 0 \leq k \leq n - m - 1; \\
        \frac{k}{n-k} = \frac{n-m}{m}, & \mbox{if } k = n - m; \\
        \frac{k}{n-k} \frac{1}{p}, & \mbox{if } n - m  + 1 \leq k \leq n - 2; \\
        n-1, & \mbox{if } k = n - 1. \\
    \end{array}
\right.    
\end{equation}

We will now consider different ranges of $k$ case by case since the exact expressions of $p_\ell^{-}$ and $p_\ell^{+}$ depend on $k$.

We begin with the case where $  0 \leq k \leq n - m - 1$. The product terms in \eqnref{eqn:ExpectedTimeLemma} can be split as follows, 
\begin{eqnarray}\label{eqn:ExpandedProducts}
\prod_{\ell=k+1}^{n-1} \frac{p_\ell^{-}}{p_\ell^{+}} &=& \left(  \prod_{\ell=k+1}^{n-m-1}\frac{p_\ell^{-}}{p_\ell^{+}} \right) \frac{p_{n-m}^{-}}{p_{n-m}^{+}} \left(\prod_{\ell=n-m+1}^{n-2}   \frac{p_\ell^{-}}{p_\ell^{+}} \right)  \frac{p_{n-1}^{-}}{p_{n-1}^{+}}.
\end{eqnarray}

Plugging in the formulas for $\frac{p_k^{-}}{p_k^{+}}$ in \eqnref{eqn:PossibleFractionValues} into 
 the product terms in \eqnref{eqn:ExpandedProducts} yields,
\begin{align*}
     \prod\limits_{\ell=k+1}^{n-m-1} \frac{p_\ell^{-}}{p_\ell^{+}} &= p^{n-m-k-1} \prod\limits_{\ell=k+1}^{n-m-1} \frac{\ell}{n-\ell}, \\
      \prod\limits_{\ell=n-m+1}^{n-2} \frac{p_\ell^{-}}{p_\ell^{+}} &= \left( \frac{1}{p} \right)^{m-2} \prod\limits_{\ell=n-m+1}^{n-2} \frac{\ell}{n-\ell}.    
\end{align*}

Therefore, for  $  0 \leq k \leq n - m - 1$ the summands in \eqnref{eqn:ExpectedTimeLemma} are
\begin{align}
    \frac{1}{p_k^+} \prod_{\ell=k+1}^{n-1} \frac{p_\ell^{-}}{p_\ell^{+}} =& \frac{1}{p_k^+}\left(  \prod_{\ell=k+1}^{n-m-1}\frac{p_\ell^{-}}{p_\ell^{+}} \right) \frac{p_{n-m}^{-}}{p_{n-m}^{+}} \left(\prod_{\ell=n-m+1}^{n-2}   \frac{p_\ell^{-}}{p_\ell^{+}} \right)  \frac{p_{n-1}^{-}}{p_{n-1}^{+}} \notag \\
=& \frac{n}{n-k} p^{n-m-k-1}  \left( \frac{1}{p} \right)^{m-2} \prod_{\ell=k+1}^{n-1} \frac{\ell}{n-\ell} \notag\\
=& p^{n-2m-k+1} \frac{n!}{(n-k)!k!} \notag \\
=& \binom{n}{k}p^{n-2m-k+1}.\label{eqn:FormulaAscent}
\end{align}

In the case where $  n - m  \leq k \leq n - 2$, by once again plugging the formulas for $\frac{p_k^{-}}{p_k^{+}}$ in \eqnref{eqn:PossibleFractionValues} into the product terms in \eqnref{eqn:ExpectedTimeLemma} we obtain,
\begin{eqnarray*}
\prod_{\ell=k+1}^{n-1} \frac{p_\ell^{-}}{p_\ell^{+}} &=& \left( \prod_{\ell=k+1}^{n-2} \frac{p_\ell^{-}}{p_\ell^{+}} \right) \frac{p_{n-1}^{-}}{p_{n-1}^{+}} \\ &=& \left( \frac{1}{p} \right)^{n-k-2} \left(\prod_{\ell=k+1}^{n-2} \frac{\ell}{n-\ell} \right) (n-1) \\
&=& \left( \frac{1}{p} \right)^{n-k-2} \frac{(n-1)!}{k!(n-(k+1))!}. 
\end{eqnarray*}

Therefore, for  $  n - m  \leq k \leq n - 2$ the summands in \eqnref{eqn:ExpectedTimeLemma} are
\begin{eqnarray}
\frac{1}{p_k^+} \prod_{\ell=k+1}^{n-1} \frac{p_\ell^{-}}{p_\ell^{+}} &=& \left( \frac{1}{p} \right)^{n-k-1} \frac{n(n-1)!}{k!(n-k)!} \notag\\
&=& \binom{n}{k} \left( \frac{1}{p} \right)^{n-k-1}.\label{eqn:FormulaDescent}    
\end{eqnarray}

Finally, by inserting \eqnref{eqn:FormulaAscent} and \eqnref{eqn:FormulaDescent} into \eqnref{eqn:ExpectedTimeLemma} we obtain,  
$$\mathbb{E}[T_{n-1}^{+}] =  p^{n-2m+1} \sum_{k=0}^{n-m-1} \binom{n}{k} p^{-k}  + p^{1-n} \sum_{k = n-m}^{n-1} \binom{n}{k} p^{k}.$$
\end{proof}

In Theorem \ref{thmformula}, we proved a general formula which will be used to derive a lower bound for the expected time of the \MAHH to arrive at the maximum of $\jump_m$ functions for two different forms of $m$.

\subsection{Case Where $m = o(\sqrt{n})$} 

We first consider the case where $m= o(\sqrt{n})$, which means that the descent is small relative to $n$. 

\begin{theorem}
If $m=o(\sqrt{n})$, then the expected runtime of the \MAHH on $\jump_m$, denoted by $T$, is
$$\mathbb{E} [T] = \Omega  \left( \frac{n^{2m-1}}{(2m-1)!}\right).$$
\end{theorem}

\begin{proof}
By only taking the term $k=n-2m+1$ in \eqnref{eqn:GeneralFormula}, we obtain, 
\begin{eqnarray}
\mathbb{E}\left[T_{n-1}^{+}\right]  & \geq & \binom{n}{n-2m+1}\notag\\
&=& \frac{n(n-1)...(n-2m+2)}{(2m-1)!}\notag\\
&=& \frac{n^{2m-1}(1-\frac{1}{n})...(1-\frac{2m-2}{n})}{(2m-1)!}\notag \\
& \geq &  \frac{n^{2m-1}(1-\frac{2m-2}{n})^{2m-1}}{(2m-1)!}\notag\\ 
&\geq& \frac{n^{2m-1}}{(2m-1)!}e^{-\frac{(2m-1)(2m-2)}{n}}\label{eqn:ExponentialBound}\\
&\underset{m = o(\sqrt{n})}{\sim} &\frac{n^{2m-1}}{(2m-1)!},\notag 
\end{eqnarray}
where we make use of the fact that $(1-x)^y \leq e^{-xy}$ in \eqnref{eqn:ExponentialBound}.
\end{proof}


\subsection{Case Where $m = \alpha n$ With $ \alpha < 0.5$}

We now consider the case where the descent has a length linear in~$n$. We shall prove that irrespective of the choice of the parameter~$p$, the runtime of the hyper-heuristics is exponential in this case.

\begin{theorem}
    The expected runtime of the \MAHH on $\jump_m$, denoted by $T$, in the case where $m$ is linear in $n$, i.e., $m = \alpha n$ with $ \alpha < 0.5$, is such that there exists $\beta > 1$ for which we have the exponential lower bound,
    $$\mathbb{E}[T] = \Omega \left(\beta^n\right).$$
\end{theorem}

\begin{proof}

By only taking the term $k=n-m$ in \eqnref{eqn:GeneralFormula}, we obtain
\begin{eqnarray*}
\mathbb{E}[T_{n-1}^{+}] &\geq& \binom{n}{n-m}p^{1-m} \\
&\geq& \binom{n}{n-m}.
\end{eqnarray*}
Stirling's formula gives 
\begin{eqnarray*}
n! &\sim& \left( \frac{n}{e} \right)^{n} \sqrt{2 \pi n}, \\
((1-\alpha)n)! &\sim& \left( \frac{(1-\alpha)n}{e} \right)^{(1-\alpha)n} \sqrt{2 \pi (1 - \alpha)n}, \\
(\alpha n)! &\sim& \left( \frac{\alpha n}{e} \right)^{\alpha n} \sqrt{2 \pi \alpha n}.
\end{eqnarray*}
Therefore, 
\begin{eqnarray*}
((1-\alpha)n)!(\alpha n)! &\sim& \alpha^{\alpha n} (1 - \alpha)^{(1 - \alpha)n} \left( \frac{n}{e} \right)^{n} 2 \pi n \sqrt{(1 - \alpha)\alpha}, \\
\frac{n!}{((1-\alpha)n)!(\alpha n )!} &\sim& C \frac{\kappa^{n}}{\sqrt{n}},
\end{eqnarray*} 
 with $\kappa = \frac{1}{ \alpha^{\alpha} (1 - \alpha)^{1 - \alpha}} > 1$.
 
Hence, for $1<\beta<\kappa$, we have
$$\mathbb{E}[T] = \Omega \left(\beta^n\right). $$
\end{proof}

\section{Upper Bound on the Runtime of \MAHH on $\jump_m$}
\label{sec:upper_bound}

We will now derive an upper bound on the runtime of the \MAHH on $\jump_m$ in the case $p =\frac{m}{n}$ to show that our lower bound in Section \ref{sec:lower_bound} is in fact optimal for this parameter choice. 
\begin{theorem}\label{thm:UpperBound}
If $p=\frac{m}{n}$, the expected runtime $T$ of the \MAHH on $\jump_m$ is 
$$\mathbb{E}\left[T\right] = O\left(n \log n + \frac{n^{2m-1}}{m!m^{m-2}}\right).$$
\end{theorem}

The remainder of this section contains the proof of Theorem \ref{thm:UpperBound}. In order to prove this result, we split the expected waiting time $T$ in two waiting times,
$$T = T_1 + T_2,$$
where $T_1$ is the time to reach the local maximum from the $0^n$ bit-string and $T_2$ is the time to reach the $1^n$ bit-string from the local maximum. 

We define the distance $d$ to the local optimum, for all $ x \in \{0,1\}^n $, by
$$d(x) = \vert n-m- \Vert x \Vert_1\vert.$$


\begin{lemma}
\label{lemma:drift}
If we denote by $(X^{(t)})_{t\geq0}$ the sequence of states of a run of the \MAHH algorithm on $\jump_m$, then while $n-m- \Vert x \Vert_1 \geq 0$, i.e., while we are on the left of the local optimum,
$$\mathbb{E}\left[d\left(X^{(t)}\right)-d\left(X^{(t+1)}\right) \,\middle|\, d\left(X^{(t)}\right)\right] \geq \frac{d\left(X^{(t)}\right)}{n}. $$
\end{lemma}

\begin{proof}
We compute
\begin{align}
\mathbb{E}&\left[d\left(X^{(t+1)}\right)\,\middle|\, d\left(X^{(t)}\right)\right] \notag \\
&= \left( d\left(X^{(t)}\right) - 1 \right) \frac{n - \Vert x \Vert_1}{n} \notag\\
&\quad + \left( d\left(X^{(t)}\right) + 1 \right) p \frac{\Vert x \Vert_1}{n}\notag \\
&\quad + d\left(X^{(t)}\right) (1-p) \frac{\Vert x \Vert_1}{n}\notag\\
&= d\left(X^{(t)}\right)+ p \frac{\Vert x \Vert_1}{n} - \frac{n - \Vert x \Vert_1}{n} \notag\\
&= d\left(X^{(t)}\right)+ (p+1) \frac{\Vert x \Vert_1}{n} - 1\notag\\
&= d\left(X^{(t)}\right)+ (p+1) \frac{n-m-d\left(X^{(t)}\right)}{n} - 1\notag\\
&\leq d\left(X^{(t)}\right)+ \left(\frac{m}{n}+1\right) \frac{n-m-d\left(X^{(t)}\right)}{n} - 1 \label{eqn:RequiresPBound}\\
&\leq d\left(X^{(t)}\right) - \frac{d\left(X^{(t)}\right)}{n} +\frac{n^2-m^2}{n^2}-1\notag\\
&\leq d\left(X^{(t)}\right) \left(1-\frac{1}{n}\right),\notag
\end{align}
where we use the fact that  $p \leq \frac{m}{n}$ in \eqnref{eqn:RequiresPBound}. Therefore we obtain
$$\mathbb{E}\left[d\left(X^{(t)}\right)-d\left(X^{(t+1)}\right) \,\middle|\, d\left(X^{(t)}\right)\right] \geq \frac{d\left(X^{(t)}\right)}{n}. $$
\end{proof}

\begin{lemma} \label{lemma:drift2}
If we denote by $(X^{(t)})_{t\geq0}$ the sequence of states of a run of the \MAHH algorithm on $\jump_m$, then while $n-m- \Vert x \Vert_1 \leq 0$, i.e., while we are on the right of the local optimum,
$$\mathbb{E}\left[d\left(X^{(t)}\right)-d\left(X^{(t+1)}\right) \,\middle|\, d\left(X^{(t)}\right)\right] \geq \frac{d\left(X^{(t)}\right)}{n}. $$
\end{lemma}

\begin{proof}
We compute
\begin{align}
\mathbb{E}&\left[d\left(X^{(t+1)}\right) \,\middle|\, d\left(X^{(t)}\right)\right]  \nonumber\\
&= \left( d\left(X^{(t)}\right) + 1 \right) p \frac{n - \Vert x \Vert_1}{n}  \nonumber\\
& \quad + \left( d\left(X^{(t)}\right) - 1 \right) \frac{\Vert x \Vert_1}{n} \label{eq:adddrift}\\
& \quad + d\left(X^{(t)}\right) (1-p) \frac{n - \Vert x \Vert_1}{n}\nonumber\\
&= d\left(X^{(t)}\right)+ p \frac{n-\Vert x \Vert_1}{n} - \frac{ \Vert x \Vert_1}{n} \nonumber\\
&= d\left(X^{(t)}\right)+ p \frac{-d\left(X^{(t)}\right)+m}{n} - \frac{d\left(X^{(t)}\right)+n-m}{n} \nonumber\\
&= d\left(X^{(t)}\right)\left(1-\frac{1}{n}\right) - p \frac{d\left(X^{(t)}\right)}{n}-\frac{n-m}{n} + \frac{pm}{n}\nonumber\\
&\leq d\left(X^{(t)}\right)\left(1-\frac{1}{n}\right) -\frac{n-m}{n} + \frac{pm}{n}\nonumber\\
&\leq d\left(X^{(t)}\right) \left(1-\frac{1}{n}\right).\nonumber
\end{align}
\end{proof}

We can now bound $T_1$. Indeed, we have, 
$$T_1 = \min \{ t :  d(X^{(t+1)}) = 0 \}.$$
By the multiplicative drift theorem (Theorem \ref{thm:mult-drift}) applied with $s_0 = d(X^{(0)}) = n-m$, $s_{\min}=1$ and $\delta = \frac{1}{n}$, we obtain
$$\mathbb{E}\left[T_1\right] = O(n \log n).$$

To bound $T_2$ we define phases of random times $(P_i)_{i \geq 0}$ when $X$ returns to $\{x^{\text{loc}},x^*\}$, where $x^{\text{loc}}$ denotes the local optimum and $x^*$ denotes the global optimum,
\begin{align*}
        P_0 &= T_1, \\
        P_{i+1} &= \min\{k \text{ : } k \geq P_{i}+1 \text{ and } X_{k} \in \{x^{\text{loc}},x^*\} \}.    
\end{align*}

Therefore, we have
$$T_2 = \sum_{i = 1}^{N} P_{i}-P_{i-1}, $$ 
where $N$ is the random number of phases until $X$ reaches $x^*$. As the $(X^{(t)})_{t\geq0}$  verifies the Markov property, the $(P_{i+1} - P_{i})_{i \geq 0}$ are independent and following the same law.

We first prove the following lemma,
\begin{lemma}
\label{lm:P}
If $p \leq \frac{m}{n}$, then we have along the running of the \MAHH algorithm on $\jump_m$,
$$\mathbb{E}\left[P_{i+1}-P_i\right] = O(m).$$
\end{lemma}

\begin{proof}
In order to prove this lemma, we will use the multiplicative drift again. 
We observe that from the local optimum three possible moves exist: 1) the random walk stays at the local optimum, 2) the random walk goes left of the local optimum, or 3) the random walk moves  right of the local optimum.

In case 1), the expected phase length is $1$.

In case 2), we can use Lemma \ref{lemma:drift} with initial position $s_0 = d(X^{(P_i+1)}) = 1$. In this case, the multiplicative drift theorem gives an expected phase length of $O(n)$.

In case 3), we use Lemma \ref{lemma:drift2} for the right part of the local optima with initial position $s_0 = d(X^{(1)}) = 1$. The multiplicative drift theorem gives an expected duration of $O(n)$.

Jointly considering these three cases allows us to establish, 
\begin{align}
\mathbb{E}\left[P_{i+1}-P_i\right] 
&= p\mathbb{E}\left[P_{i+1}-P_i \,\middle|\, \lVert X^{(P_i+1)}\rVert_1<\lVert x^{\text{loc}}\rVert_1\right] \notag\\ 
& \quad + (1 - p)\mathbb{E}\left[P_{i+1}-P_i \mid \lVert X^{(P_i+1)} \rVert_1=\lVert x^{\text{loc}}\Vert_1 \right] \notag \\ 
& \quad + p\mathbb{E}\left[P_{i+1}-P_i \mid  \lVert X^{(P_i+1)}\rVert_1>\lVert x^{\text{loc}}\rVert_1\right] \notag\\
&= pO(n) + (1-p) + pO(n) \notag\\
&= O(m),\label{eqn:Om}
\end{align}
where in \eqnref{eqn:Om} we make use of $p \leq \frac{m}{n}.$
\end{proof}

The path starting at the local optima and going straight to the global optimum has a probability of \[
\prod_{k=1}^{m} \frac{k}{n}p^{m-1} = \frac{m!}{n^m}p^{m-1}.
\]
Therefore, the probability that the phase ends in $x^*$ is 
\begin{align*}
\mathbb{P}\left[X_{P_{i+1}} = x^*\right] \geq \frac{m!}{n^m}p^{m-1}.
\end{align*}

\begin{lemma}
\label{lm:N}
If $ p\geq \frac{m}{n}$, then we have along the running of the \MAHH algorithm on $\jump_m$,
    $$\mathbb{E}\left[N\right] = O\left(\frac{n^{2m-1}}{m!m^{m-1}}\right).$$
\end{lemma}

\begin{proof}
For each new phase $P_i$, for $i \geq 0$, the algorithm has the same probability of reaching the global maximum during the next phase. Therefore, the random variable $N$ follows a geometric distribution with parameter 
\[
\mathbb{P}\left[X_{P_{i+1}} = x^*\right] \geq \frac{m!}{n^m}p^{m-1}.
\]
Hence, if $p\geq \frac{m}{n}$,
$$\mathbb{E}\left[N\right] = \frac{1}{\mathbb{P}\left[X_{P_{i+1}} = x^*\right]} \leq \frac{n^{2m-1}}{m!m^{m-1}}.$$
\end{proof}

We can now combine the results of Lemmas~\ref{lm:P} and \ref{lm:N} and Wald's formula (Theorem \ref{thm:wald}), to obtain that if $p=\frac{m}{n},$

\begin{eqnarray*}
\mathbb{E}\left[T_2\right] &=& \mathbb{E}\left[N\right]\mathbb{E}\left[P_1-P_0\right]\\
&=&O(m\mathbb{E}\left[N\right])\\
&=& O\left(\frac{n^{2m-1}}{m!m^{m-2}}\right).
\end{eqnarray*}

This proves that the \MAHH heuristic performs much worse than the elitist heuristic which has an expected time of $O(n^m)$.

\section{Using Global Mutation} 

We have just proven in Section \ref{sec:upper_bound} that the \MAHH with $p = \frac{m}{n}$ has a runtime of $O\left( \frac{n^{2m-1}}{m!m^{m-2}}\right)$. For $m = o(\sqrt n)$, this is significantly larger than the $\Theta(n^m)$ runtime of many evolutionary algorithms. Noting this performance gap between the \oea and the \MAHH, we propose to use the \MAHH with bit-wise mutation, the variation operator of the \oea, instead of one-bit flips. As we will see, this reduces the runtime to essentially the runtime of the \oea (apart from a factor of~$m$). Our runtime guarantee is never worse than the one of the classic \MAHH, so in the (extremal) cases where the classic \MAHH is superior to the \oea, this algorithms is as well. We note that it may appear natural that an algorithm equipped with two mechanisms to leave local optima shows such a best-of-two-worlds behavior, but since the global mutation operator opens also many other search trajectories, it is not immediately obvious that a combination of two operators results in the minimum performance of the two individual performances. In fact, in our proofs we will see that the global mutation operator also leads to (small) negative effects. We discuss this in more detail in the sketch of the proof. 

\begin{algorithm}[t]
\caption{\MAHH with Global Mutation} \label{alg:comb}
\KwData{Choose $x \in \{0,1 \}^n$ uniformly at random}
\While{termination criteria not satisfied}{
  $x^{\prime} \gets\text{flip each bit of } x \text{ with probability } \frac{1}{n}$\;
  $\text{ACC} \gets 
  \begin{cases}
      \text{ALLMOVES}, & \text{with probability } p; \\
      \text{ONLYIMPROVING}, & \text{otherwise};\\
  \end{cases}$\\
  \If{$\text{ACC}(x,x^{\prime})$}{
    $x \gets x^{\prime}$\;
  }
}
\end{algorithm}


We show the following runtime guarantee for the \MAHH with global mutation (bit-wise mutation with mutation rate~$\frac 1n$).

\begin{theorem}\label{thm:combination}
If $p=\frac{m}{8 e n}$, the expected runtime $T$ of the \MAHH with global mutation on $\jump_m$ verifies 
$$\mathbb{E}\left[T\right] = O\left(n \log n + m\min\left(en^m, \frac{8^{m-1}(en)^{2m-1}}{m!m^{m-1}}\right)\right).$$
\end{theorem}

In order to prove this result, we will follow the same general approach via Wald's equation as in Section \ref{sec:upper_bound}. Again, a phase is the time interval starting from the current solution being on the local optimum to the current solution for the first time being again on the local optimum or being the global optimum.

As before, the expected number of phases is the reciprocal of the probability that a phase ends on the global optimum. Since this is the sum of the probabilities that the global optimum is reached with a mutation step going right from the local to the global optimum and the probability that the \MAHH uses the trajectory through the valley of low fitness, we indeed have the best-of-two-worlds effect for the expected number of phases.

The difference to our situation in Section \ref{sec:upper_bound} lies in the way our algorithm returns to the local optimum. Of course, again, the typical phase (of length one) is that the algorithm just stays on the local optimum, either because the offspring equals the parent (this can happen now that we use bit-wise mutation) or because the new solution, which is worse than the parent unless it is the global optimum, is not accepted (which happens with high probability $1-p$). 

The different behavior arises when the current solution is in the fitness valley, that is, when $\|x\|_1$ is strictly between $n-m$ and~$n$. For the classic \MAHH, in this case we have a strong additive drift towards the local optimum, see~\eqref{eq:adddrift} (this strong drift was not fully exploited in the remainder to allow a uniform analysis of the cases that the algorithm is to the left and to the right of the local optimum, this is why this drift does not appear in the statement of Lemma~\ref{lemma:drift2}). This constant drift implies that it takes only constant time to reach the local optimum when in the fitness valley at constant distance to the local optimum.

Now that we use global mutation, it can happen that the algorithm flips several one-bits to zero and thereby goes from the area right of the local optimum to the left of the local optimum. There, a much smaller drift to the local optimum is present, leading to times of order at least $\Omega(n/m)$ to reach the local optimum. We note that this event of moving from the right to the left of the local optimum is not rare, but happens in fact with constant probability (simply, because with constant probability several bits are flipped and here, with constant probability, all these bits are one). Consequently, now the time to regain the local optimum when in constant distance to the right of the local optimum is at least $\Omega(n/m)$, as opposed to $O(1)$ with local mutation. This shows why it is not obvious that a combination of two mechanisms to leave local optima gives a best-of-two-worlds behavior. 

To nevertheless show a best-of-two-worlds phenomenon, we design a suitable potential function. It accounts for distances from the local optimum stemming from the fitness valley by a constant factor $\alpha$ larger than distances on the other side. With a careful analysis, estimating in particular the effect of changing the side, we show that this potential function still shows a multiplicative drift behavior, and this gives us the same $O(m)$ expected length of a phase as in Section~\ref{sec:upper_bound}.

\ifthenelse{\boolean{arxiv}}
{
For reasons of space, the formal proof had to be moved into the appendix of this work.
}
{
For reasons of space, the formal proof had to be omitted in this paper. It can be found in the preprint~\cite{DoerrDLS23arxiv}.
}

\section{Conclusion}\label{sec:conclusion}

In this work, we conducted a precise runtime analysis of the \MAHH on the \jump benchmark, the most prominent multimodal benchmark in the theory of randomized search heuristics. Our main result is that the \MAHH on the \jump  benchmark does not exhibit the extremely positive performance on the \cliff benchmark~\cite{LissovoiOW23}, but instead has a runtime that for many jump sizes, in particular, the more relevant small ones, is drastically above the runtime of many evolutionary algorithms. This could indicate that the \cliff benchmark is a too optimistic model for local optima in heuristic search. 

On the positive side, we propose to use the \MAHH with the global bit-wise mutation operator common in evolutionary algorithms and prove the non-obvious result that this leads to a runtime which is essentially the minimum of the runtimes of the classic \MAHH and the \oea. 

Since we observed this best-of-two-worlds effect so far only on the \jump benchmark, more research in this direction is clearly needed. We note that in general it is little understood how evolutionary algorithms and other randomized search heuristics profit from non-elitism. So, also more research on this broader topic would be highly interesting.
		
\begin{acks}
This work was supported by a public grant as part of the
Investissements d'avenir project, reference ANR-11-LABX-0056-LMH,
LabEx LMH.
\end{acks}

\bibliographystyle{ACM-Reference-Format}
\bibliography{ich_master,alles_ea_master,rest}


\begin{thebibliography}{56}


\ifx \showCODEN    \undefined \def \showCODEN     #1{\unskip}     \fi
\ifx \showDOI      \undefined \def \showDOI       #1{#1}\fi
\ifx \showISBNx    \undefined \def \showISBNx     #1{\unskip}     \fi
\ifx \showISBNxiii \undefined \def \showISBNxiii  #1{\unskip}     \fi
\ifx \showISSN     \undefined \def \showISSN      #1{\unskip}     \fi
\ifx \showLCCN     \undefined \def \showLCCN      #1{\unskip}     \fi
\ifx \shownote     \undefined \def \shownote      #1{#1}          \fi
\ifx \showarticletitle \undefined \def \showarticletitle #1{#1}   \fi
\ifx \showURL      \undefined \def \showURL       {\relax}        \fi
\providecommand\bibfield[2]{#2}
\providecommand\bibinfo[2]{#2}
\providecommand\natexlab[1]{#1}
\providecommand\showeprint[2][]{arXiv:#2}

\bibitem[Alanazi and Lehre(2014)]%
        {AlanaziL14}
\bibfield{author}{\bibinfo{person}{Fawaz Alanazi} {and}
  \bibinfo{person}{Per~Kristian Lehre}.} \bibinfo{year}{2014}\natexlab{}.
\newblock \showarticletitle{Runtime analysis of selection hyper-heuristics with
  classical learning mechanisms}. In \bibinfo{booktitle}{\emph{Congress on
  Evolutionary Computation, {CEC} 2014}}. \bibinfo{publisher}{IEEE},
  \bibinfo{pages}{2515--2523}.
\newblock


\bibitem[Antipov et~al\mbox{.}(2021)]%
        {AntipovBD21gecco}
\bibfield{author}{\bibinfo{person}{Denis Antipov}, \bibinfo{person}{Maxim
  Buzdalov}, {and} \bibinfo{person}{Benjamin Doerr}.}
  \bibinfo{year}{2021}\natexlab{}.
\newblock \showarticletitle{Lazy parameter tuning and control: choosing all
  parameters randomly from a power-law distribution}. In
  \bibinfo{booktitle}{\emph{Genetic and Evolutionary Computation Conference,
  GECCO 2021}}. \bibinfo{publisher}{{ACM}}, \bibinfo{pages}{1115--1123}.
\newblock


\bibitem[Antipov and Doerr(2020)]%
        {AntipovD20ppsn}
\bibfield{author}{\bibinfo{person}{Denis Antipov} {and}
  \bibinfo{person}{Benjamin Doerr}.} \bibinfo{year}{2020}\natexlab{}.
\newblock \showarticletitle{Runtime analysis of a heavy-tailed ${(1+(\lambda,
  \lambda))}$ genetic algorithm on jump functions}. In
  \bibinfo{booktitle}{\emph{Parallel Problem Solving From Nature, PPSN 2020,
  Part~II}}. \bibinfo{publisher}{Springer}, \bibinfo{pages}{545--559}.
\newblock


\bibitem[Antipov et~al\mbox{.}(2022)]%
        {AntipovDK22}
\bibfield{author}{\bibinfo{person}{Denis Antipov}, \bibinfo{person}{Benjamin
  Doerr}, {and} \bibinfo{person}{Vitalii Karavaev}.}
  \bibinfo{year}{2022}\natexlab{}.
\newblock \showarticletitle{A rigorous runtime analysis of the ${(1 +
  (\lambda,\lambda))}$ {GA} on Jump functions}.
\newblock \bibinfo{journal}{\emph{Algorithmica}}  \bibinfo{volume}{84}
  (\bibinfo{year}{2022}), \bibinfo{pages}{1573--1602}.
\newblock


\bibitem[Auger and Doerr(2011)]%
        {AugerD11}
\bibfield{editor}{\bibinfo{person}{Anne Auger} {and} \bibinfo{person}{Benjamin
  Doerr}} (Eds.). \bibinfo{year}{2011}\natexlab{}.
\newblock \bibinfo{booktitle}{\emph{Theory of Randomized Search Heuristics}}.
\newblock \bibinfo{publisher}{World Scientific Publishing}.
\newblock


\bibitem[Bambury et~al\mbox{.}(2021)]%
        {BamburyBD21}
\bibfield{author}{\bibinfo{person}{Henry Bambury}, \bibinfo{person}{Antoine
  Bultel}, {and} \bibinfo{person}{Benjamin Doerr}.}
  \bibinfo{year}{2021}\natexlab{}.
\newblock \showarticletitle{Generalized jump functions}. In
  \bibinfo{booktitle}{\emph{Genetic and Evolutionary Computation Conference,
  GECCO 2021}}. \bibinfo{publisher}{{ACM}}, \bibinfo{pages}{1124--1132}.
\newblock


\bibitem[Benbaki et~al\mbox{.}(2021)]%
        {BenbakiBD21}
\bibfield{author}{\bibinfo{person}{Riade Benbaki}, \bibinfo{person}{Ziyad
  Benomar}, {and} \bibinfo{person}{Benjamin Doerr}.}
  \bibinfo{year}{2021}\natexlab{}.
\newblock \showarticletitle{A rigorous runtime analysis of the
  2-{MMAS}$_{\mathrm{ib}}$ on jump functions: ant colony optimizers can cope
  well with local optima}. In \bibinfo{booktitle}{\emph{Genetic and
  Evolutionary Computation Conference, GECCO 2021}}.
  \bibinfo{publisher}{{ACM}}, \bibinfo{pages}{4--13}.
\newblock


\bibitem[Bian et~al\mbox{.}({[n.\,d.]})]%
        {BianZLQ23}
\bibfield{author}{\bibinfo{person}{Chao Bian}, \bibinfo{person}{Yawen Zhou},
  \bibinfo{person}{Miqing Li}, {and} \bibinfo{person}{Chao Qian}.}
  \bibinfo{year}{[n.\,d.]}\natexlab{}.
\newblock \showarticletitle{Stochastic population update can provably be
  helpful in multi-objective evolutionary algorithms}. In
  \bibinfo{booktitle}{\emph{International Joint Conference on Artificial
  Intelligence, IJCAI 2023}}.
\newblock
\newblock
\shownote{to appear}.


\bibitem[Burke et~al\mbox{.}(2013)]%
        {BurkeGHKOOQ13}
\bibfield{author}{\bibinfo{person}{Edmund~K. Burke}, \bibinfo{person}{Michel
  Gendreau}, \bibinfo{person}{Matthew~R. Hyde}, \bibinfo{person}{Graham
  Kendall}, \bibinfo{person}{Gabriela Ochoa}, \bibinfo{person}{Ender
  {\"{O}}zcan}, {and} \bibinfo{person}{Rong Qu}.}
  \bibinfo{year}{2013}\natexlab{}.
\newblock \showarticletitle{Hyper-heuristics: a survey of the state of the
  art}.
\newblock \bibinfo{journal}{\emph{Journal of the Operational Research Society}}
   \bibinfo{volume}{64} (\bibinfo{year}{2013}), \bibinfo{pages}{1695--1724}.
\newblock


\bibitem[Buzdalov et~al\mbox{.}(2016)]%
        {BuzdalovDK16}
\bibfield{author}{\bibinfo{person}{Maxim Buzdalov}, \bibinfo{person}{Benjamin
  Doerr}, {and} \bibinfo{person}{Mikhail Kever}.}
  \bibinfo{year}{2016}\natexlab{}.
\newblock \showarticletitle{The unrestricted black-box complexity of jump
  functions}.
\newblock \bibinfo{journal}{\emph{Evolutionary Computation}}
  \bibinfo{volume}{24} (\bibinfo{year}{2016}), \bibinfo{pages}{719--744}.
\newblock


\bibitem[Corus et~al\mbox{.}(2020)]%
        {CorusOY20}
\bibfield{author}{\bibinfo{person}{Dogan Corus}, \bibinfo{person}{Pietro~S.
  Oliveto}, {and} \bibinfo{person}{Donya Yazdani}.}
  \bibinfo{year}{2020}\natexlab{}.
\newblock \showarticletitle{When hypermutations and ageing enable artificial
  immune systems to outperform evolutionary algorithms}.
\newblock \bibinfo{journal}{\emph{Theoretical Computer Science}}
  \bibinfo{volume}{832} (\bibinfo{year}{2020}), \bibinfo{pages}{166--185}.
\newblock


\bibitem[Dang et~al\mbox{.}(2016)]%
        {DangFKKLOSS16}
\bibfield{author}{\bibinfo{person}{Duc{-}Cuong Dang}, \bibinfo{person}{Tobias
  Friedrich}, \bibinfo{person}{Timo K{\"{o}}tzing}, \bibinfo{person}{Martin~S.
  Krejca}, \bibinfo{person}{Per~Kristian Lehre}, \bibinfo{person}{Pietro~S.
  Oliveto}, \bibinfo{person}{Dirk Sudholt}, {and} \bibinfo{person}{Andrew~M.
  Sutton}.} \bibinfo{year}{2016}\natexlab{}.
\newblock \showarticletitle{Escaping local optima with diversity mechanisms and
  crossover}. In \bibinfo{booktitle}{\emph{Genetic and Evolutionary Computation
  Conference, GECCO 2016}}. \bibinfo{publisher}{{ACM}},
  \bibinfo{pages}{645--652}.
\newblock


\bibitem[Dang et~al\mbox{.}(2018)]%
        {DangFKKLOSS18}
\bibfield{author}{\bibinfo{person}{Duc{-}Cuong Dang}, \bibinfo{person}{Tobias
  Friedrich}, \bibinfo{person}{Timo K{\"{o}}tzing}, \bibinfo{person}{Martin~S.
  Krejca}, \bibinfo{person}{Per~Kristian Lehre}, \bibinfo{person}{Pietro~S.
  Oliveto}, \bibinfo{person}{Dirk Sudholt}, {and} \bibinfo{person}{Andrew~M.
  Sutton}.} \bibinfo{year}{2018}\natexlab{}.
\newblock \showarticletitle{Escaping local optima using crossover with emergent
  diversity}.
\newblock \bibinfo{journal}{\emph{{IEEE} Transactions on Evolutionary
  Computation}}  \bibinfo{volume}{22} (\bibinfo{year}{2018}),
  \bibinfo{pages}{484--497}.
\newblock


\bibitem[Doerr(2021)]%
        {Doerr21cgajump}
\bibfield{author}{\bibinfo{person}{Benjamin Doerr}.}
  \bibinfo{year}{2021}\natexlab{}.
\newblock \showarticletitle{The runtime of the compact genetic algorithm on
  {J}ump functions}.
\newblock \bibinfo{journal}{\emph{Algorithmica}}  \bibinfo{volume}{83}
  (\bibinfo{year}{2021}), \bibinfo{pages}{3059--3107}.
\newblock


\bibitem[Doerr(2022)]%
        {Doerr22}
\bibfield{author}{\bibinfo{person}{Benjamin Doerr}.}
  \bibinfo{year}{2022}\natexlab{}.
\newblock \showarticletitle{Does comma selection help to cope with local
  optima?}
\newblock \bibinfo{journal}{\emph{Algorithmica}}  \bibinfo{volume}{84}
  (\bibinfo{year}{2022}), \bibinfo{pages}{1659--1693}.
\newblock


\bibitem[Doerr and Doerr(2020)]%
        {DoerrD20bookchapter}
\bibfield{author}{\bibinfo{person}{Benjamin Doerr} {and}
  \bibinfo{person}{Carola Doerr}.} \bibinfo{year}{2020}\natexlab{}.
\newblock \showarticletitle{Theory of parameter control for discrete black-box
  optimization: provable performance gains through dynamic parameter choices}.
\newblock In \bibinfo{booktitle}{\emph{Theory of Evolutionary Computation:
  Recent Developments in Discrete Optimization}},
  \bibfield{editor}{\bibinfo{person}{Benjamin Doerr} {and}
  \bibinfo{person}{Frank Neumann}} (Eds.). \bibinfo{publisher}{Springer},
  \bibinfo{pages}{271--321}.
\newblock
\newblock
\shownote{Also available at \url{https://arxiv.org/abs/1804.05650}}.


\bibitem[Doerr et~al\mbox{.}(2023a)]%
        {DoerrDLS23}
\bibfield{author}{\bibinfo{person}{Benjamin Doerr}, \bibinfo{person}{Arthur
  Dremaux}, \bibinfo{person}{Johannes Lutzeyer}, {and}
  \bibinfo{person}{Aur\'elien Stumpf}.} \bibinfo{year}{2023}\natexlab{a}.
\newblock \showarticletitle{How the move acceptance hyper-heuristic copes with
  local optima: drastic differences between jumps and cliffs}. In
  \bibinfo{booktitle}{\emph{Genetic and Evolutionary Computation Conference,
  GECCO 2023}}. \bibinfo{publisher}{{ACM}}.
\newblock
\newblock
\shownote{To appear}.


\bibitem[Doerr et~al\mbox{.}(2023b)]%
        {DoerrEJK23arxiv}
\bibfield{author}{\bibinfo{person}{Benjamin Doerr}, \bibinfo{person}{Aymen
  Echarghaoui}, \bibinfo{person}{Mohammed Jamal}, {and}
  \bibinfo{person}{Martin~S. Krejca}.} \bibinfo{year}{2023}\natexlab{b}.
\newblock \showarticletitle{Lasting Diversity and Superior Runtime Guarantees
  for the $(\mu+1)$ Genetic Algorithm}.
\newblock \bibinfo{journal}{\emph{CoRR}}  \bibinfo{volume}{abs/2302.12570}
  (\bibinfo{year}{2023}).
\newblock
\showeprint[arXiv]{2302.12570}


\bibitem[Doerr et~al\mbox{.}(2023c)]%
        {DoerrERW23}
\bibfield{author}{\bibinfo{person}{Benjamin Doerr}, \bibinfo{person}{Taha {El
  Ghazi El Houssaini}}, \bibinfo{person}{Amirhossein Rajabi}, {and}
  \bibinfo{person}{Carsten Witt}.} \bibinfo{year}{2023}\natexlab{c}.
\newblock \showarticletitle{How well does the Metropolis algorithm cope with
  local optima?}. In \bibinfo{booktitle}{\emph{Genetic and Evolutionary
  Computation Conference, GECCO 2023}}. \bibinfo{publisher}{{ACM}}.
\newblock
\newblock
\shownote{To appear}.


\bibitem[Doerr et~al\mbox{.}(2012)]%
        {DoerrJW12algo}
\bibfield{author}{\bibinfo{person}{Benjamin Doerr}, \bibinfo{person}{Daniel
  Johannsen}, {and} \bibinfo{person}{Carola Winzen}.}
  \bibinfo{year}{2012}\natexlab{}.
\newblock \showarticletitle{Multiplicative drift analysis}.
\newblock \bibinfo{journal}{\emph{Algorithmica}}  \bibinfo{volume}{64}
  (\bibinfo{year}{2012}), \bibinfo{pages}{673--697}.
\newblock


\bibitem[Doerr et~al\mbox{.}(2017)]%
        {DoerrLMN17}
\bibfield{author}{\bibinfo{person}{Benjamin Doerr}, \bibinfo{person}{Huu~Phuoc
  Le}, \bibinfo{person}{R\'egis Makhmara}, {and} \bibinfo{person}{Ta~Duy
  Nguyen}.} \bibinfo{year}{2017}\natexlab{}.
\newblock \showarticletitle{Fast genetic algorithms}. In
  \bibinfo{booktitle}{\emph{Genetic and Evolutionary Computation Conference,
  GECCO 2017}}. \bibinfo{publisher}{{ACM}}, \bibinfo{pages}{777--784}.
\newblock


\bibitem[Doerr et~al\mbox{.}(2018)]%
        {DoerrLOW18}
\bibfield{author}{\bibinfo{person}{Benjamin Doerr}, \bibinfo{person}{Andrei
  Lissovoi}, \bibinfo{person}{Pietro~S. Oliveto}, {and}
  \bibinfo{person}{John~Alasdair Warwicker}.} \bibinfo{year}{2018}\natexlab{}.
\newblock \showarticletitle{On the runtime analysis of selection
  hyper-heuristics with adaptive learning periods}. In
  \bibinfo{booktitle}{\emph{Genetic and Evolutionary Computation Conference,
  GECCO 2018}}. \bibinfo{publisher}{ACM}, \bibinfo{pages}{1015--1022}.
\newblock


\bibitem[Doerr and Neumann(2020)]%
        {DoerrN20}
\bibfield{editor}{\bibinfo{person}{Benjamin Doerr} {and} \bibinfo{person}{Frank
  Neumann}} (Eds.). \bibinfo{year}{2020}\natexlab{}.
\newblock \bibinfo{booktitle}{\emph{Theory of Evolutionary Computation---Recent
  Developments in Discrete Optimization}}.
\newblock \bibinfo{publisher}{Springer}.
\newblock
\newblock
\shownote{Also available at
  \url{http://www.lix.polytechnique.fr/Labo/Benjamin.Doerr/doerr_neumann_book.html}}.


\bibitem[Doerr and Qu(2023a)]%
        {DoerrQ23tec}
\bibfield{author}{\bibinfo{person}{Benjamin Doerr} {and}
  \bibinfo{person}{Zhongdi Qu}.} \bibinfo{year}{2023}\natexlab{a}.
\newblock \showarticletitle{A first runtime analysis of the {NSGA-II} on a
  multimodal problem}.
\newblock \bibinfo{journal}{\emph{Transactions on Evolutionary Computation}}
  (\bibinfo{year}{2023}).
\newblock
\newblock
\shownote{\url{https://doi.org/10.1109/TEVC.2023.3250552}}.


\bibitem[Doerr and Qu(2023b)]%
        {DoerrQ23LB}
\bibfield{author}{\bibinfo{person}{Benjamin Doerr} {and}
  \bibinfo{person}{Zhongdi Qu}.} \bibinfo{year}{2023}\natexlab{b}.
\newblock \showarticletitle{From understanding the population dynamics of the
  {NSGA-II} to the first proven lower bounds}. In
  \bibinfo{booktitle}{\emph{Conference on Artificial Intelligence, {AAAI}
  2023}}. \bibinfo{publisher}{{AAAI} Press}.
\newblock
\newblock
\shownote{To appear}.


\bibitem[Doerr and Qu(2023c)]%
        {DoerrQ23crossover}
\bibfield{author}{\bibinfo{person}{Benjamin Doerr} {and}
  \bibinfo{person}{Zhongdi Qu}.} \bibinfo{year}{2023}\natexlab{c}.
\newblock \showarticletitle{Runtime analysis for the NSGA-II: Provable
  speed-ups from crossover}. In \bibinfo{booktitle}{\emph{Conference on
  Artificial Intelligence, {AAAI} 2023}}. \bibinfo{publisher}{{AAAI} Press}.
\newblock
\newblock
\shownote{To appear}.


\bibitem[Doerr and Rajabi(2023)]%
        {DoerrR23}
\bibfield{author}{\bibinfo{person}{Benjamin Doerr} {and}
  \bibinfo{person}{Amirhossein Rajabi}.} \bibinfo{year}{2023}\natexlab{}.
\newblock \showarticletitle{Stagnation detection meets fast mutation}.
\newblock \bibinfo{journal}{\emph{Theoretical Computer Science}}
  \bibinfo{volume}{946} (\bibinfo{year}{2023}), \bibinfo{pages}{113670}.
\newblock


\bibitem[Doerr and Zheng(2021)]%
        {DoerrZ21aaai}
\bibfield{author}{\bibinfo{person}{Benjamin Doerr} {and}
  \bibinfo{person}{Weijie Zheng}.} \bibinfo{year}{2021}\natexlab{}.
\newblock \showarticletitle{Theoretical analyses of multi-objective
  evolutionary algorithms on multi-modal objectives}. In
  \bibinfo{booktitle}{\emph{Conference on Artificial Intelligence, {AAAI}
  2021}}. \bibinfo{publisher}{{AAAI} Press}, \bibinfo{pages}{12293--12301}.
\newblock


\bibitem[Droste et~al\mbox{.}(2000)]%
        {DrosteJW00}
\bibfield{author}{\bibinfo{person}{Stefan Droste}, \bibinfo{person}{Thomas
  Jansen}, {and} \bibinfo{person}{Ingo Wegener}.}
  \bibinfo{year}{2000}\natexlab{}.
\newblock \showarticletitle{Dynamic parameter control in simple evolutionary
  algorithms}. In \bibinfo{booktitle}{\emph{Foundations of Genetic Algorithms,
  FOGA 2000}}. \bibinfo{publisher}{Morgan Kaufmann}, \bibinfo{pages}{275--294}.
\newblock


\bibitem[Droste et~al\mbox{.}(2002)]%
        {DrosteJW02}
\bibfield{author}{\bibinfo{person}{Stefan Droste}, \bibinfo{person}{Thomas
  Jansen}, {and} \bibinfo{person}{Ingo Wegener}.}
  \bibinfo{year}{2002}\natexlab{}.
\newblock \showarticletitle{On the analysis of the (1+1) evolutionary
  algorithm}.
\newblock \bibinfo{journal}{\emph{Theoretical Computer Science}}
  \bibinfo{volume}{276} (\bibinfo{year}{2002}), \bibinfo{pages}{51--81}.
\newblock


\bibitem[Friedrich et~al\mbox{.}(2016)]%
        {FriedrichKKNNS16}
\bibfield{author}{\bibinfo{person}{Tobias Friedrich}, \bibinfo{person}{Timo
  K{\"{o}}tzing}, \bibinfo{person}{Martin~S. Krejca}, \bibinfo{person}{Samadhi
  Nallaperuma}, \bibinfo{person}{Frank Neumann}, {and} \bibinfo{person}{Martin
  Schirneck}.} \bibinfo{year}{2016}\natexlab{}.
\newblock \showarticletitle{Fast building block assembly by majority vote
  crossover}. In \bibinfo{booktitle}{\emph{Genetic and Evolutionary Computation
  Conference, GECCO 2016}}. \bibinfo{publisher}{{ACM}},
  \bibinfo{pages}{661--668}.
\newblock


\bibitem[Friedrich et~al\mbox{.}(2022)]%
        {FriedrichKKR22}
\bibfield{author}{\bibinfo{person}{Tobias Friedrich}, \bibinfo{person}{Timo
  K{\"{o}}tzing}, \bibinfo{person}{Martin~S. Krejca}, {and}
  \bibinfo{person}{Amirhossein Rajabi}.} \bibinfo{year}{2022}\natexlab{}.
\newblock \showarticletitle{Escaping local optima with local search: {A}
  theory-driven discussion}. In \bibinfo{booktitle}{\emph{Parallel Problem
  Solving from Nature, {PPSN} 2022, Part {II}}},
  \bibfield{editor}{\bibinfo{person}{G{\"{u}}nter Rudolph},
  \bibinfo{person}{Anna~V. Kononova}, \bibinfo{person}{Hern{\'{a}}n~E.
  Aguirre}, \bibinfo{person}{Pascal Kerschke}, \bibinfo{person}{Gabriela
  Ochoa}, {and} \bibinfo{person}{Tea Tusar}} (Eds.).
  \bibinfo{publisher}{Springer}, \bibinfo{pages}{442--455}.
\newblock


\bibitem[Giel and Wegener(2003)]%
        {GielW03}
\bibfield{author}{\bibinfo{person}{Oliver Giel} {and} \bibinfo{person}{Ingo
  Wegener}.} \bibinfo{year}{2003}\natexlab{}.
\newblock \showarticletitle{Evolutionary algorithms and the maximum matching
  problem}. In \bibinfo{booktitle}{\emph{Symposium on Theoretical Aspects of
  Computer Science, STACS 2003}}. \bibinfo{publisher}{Springer},
  \bibinfo{pages}{415--426}.
\newblock


\bibitem[Hasen{\"{o}}hrl and Sutton(2018)]%
        {HasenohrlS18}
\bibfield{author}{\bibinfo{person}{V{\'{a}}clav Hasen{\"{o}}hrl} {and}
  \bibinfo{person}{Andrew~M. Sutton}.} \bibinfo{year}{2018}\natexlab{}.
\newblock \showarticletitle{On the runtime dynamics of the compact genetic
  algorithm on jump functions}. In \bibinfo{booktitle}{\emph{Genetic and
  Evolutionary Computation Conference, {GECCO} 2018}}.
  \bibinfo{publisher}{{ACM}}, \bibinfo{pages}{967--974}.
\newblock


\bibitem[Hevia~Fajardo and Sudholt(2021)]%
        {FajardoS21foga}
\bibfield{author}{\bibinfo{person}{Mario~Alejandro Hevia~Fajardo} {and}
  \bibinfo{person}{Dirk Sudholt}.} \bibinfo{year}{2021}\natexlab{}.
\newblock \showarticletitle{Self-adjusting offspring population sizes
  outperform fixed parameters on the cliff function}. In
  \bibinfo{booktitle}{\emph{Foundations of Genetic Algorithms, FOGA 2021}}.
  \bibinfo{publisher}{{ACM}}, \bibinfo{pages}{5:1--5:15}.
\newblock


\bibitem[J{\"a}gersk{\"u}pper and Storch(2007)]%
        {JagerskupperS07}
\bibfield{author}{\bibinfo{person}{Jens J{\"a}gersk{\"u}pper} {and}
  \bibinfo{person}{Tobias Storch}.} \bibinfo{year}{2007}\natexlab{}.
\newblock \showarticletitle{When the plus strategy outperforms the comma
  strategy and when not}. In \bibinfo{booktitle}{\emph{Foundations of
  Computational Intelligence, FOCI 2007}}. \bibinfo{publisher}{IEEE},
  \bibinfo{pages}{25--32}.
\newblock


\bibitem[Jansen(2013)]%
        {Jansen13}
\bibfield{author}{\bibinfo{person}{Thomas Jansen}.}
  \bibinfo{year}{2013}\natexlab{}.
\newblock \bibinfo{booktitle}{\emph{Analyzing Evolutionary Algorithms -- The
  Computer Science Perspective}}.
\newblock \bibinfo{publisher}{Springer}.
\newblock


\bibitem[Jansen(2015)]%
        {Jansen15}
\bibfield{author}{\bibinfo{person}{Thomas Jansen}.}
  \bibinfo{year}{2015}\natexlab{}.
\newblock \showarticletitle{On the black-box complexity of example functions:
  the real jump function}. In \bibinfo{booktitle}{\emph{Foundations of Genetic
  Algorithms, FOGA 2015}}. \bibinfo{publisher}{{ACM}}, \bibinfo{pages}{16--24}.
\newblock


\bibitem[Jansen and Wegener(2001)]%
        {JansenW01}
\bibfield{author}{\bibinfo{person}{Thomas Jansen} {and} \bibinfo{person}{Ingo
  Wegener}.} \bibinfo{year}{2001}\natexlab{}.
\newblock \showarticletitle{Evolutionary algorithms - how to cope with plateaus
  of constant fitness and when to reject strings of the same fitness}.
\newblock \bibinfo{journal}{\emph{{IEEE} Transactions on Evolutionary
  Computation}}  \bibinfo{volume}{5} (\bibinfo{year}{2001}),
  \bibinfo{pages}{589--599}.
\newblock


\bibitem[K{\"{o}}tzing et~al\mbox{.}(2011)]%
        {KotzingST11}
\bibfield{author}{\bibinfo{person}{Timo K{\"{o}}tzing}, \bibinfo{person}{Dirk
  Sudholt}, {and} \bibinfo{person}{Madeleine Theile}.}
  \bibinfo{year}{2011}\natexlab{}.
\newblock \showarticletitle{How crossover helps in pseudo-{B}oolean
  optimization}. In \bibinfo{booktitle}{\emph{Genetic and Evolutionary
  Computation Conference, {GECCO} 2011}}. \bibinfo{publisher}{{ACM}},
  \bibinfo{pages}{989--996}.
\newblock


\bibitem[Lehre and {\"{O}}zcan(2013)]%
        {LehreO13}
\bibfield{author}{\bibinfo{person}{Per~Kristian Lehre} {and}
  \bibinfo{person}{Ender {\"{O}}zcan}.} \bibinfo{year}{2013}\natexlab{}.
\newblock \showarticletitle{A runtime analysis of simple hyper-heuristics: to
  mix or not to mix operators}. In \bibinfo{booktitle}{\emph{Foundations of
  Genetic Algorithms, {FOGA} 2013}}. \bibinfo{publisher}{{ACM}},
  \bibinfo{pages}{97--104}.
\newblock


\bibitem[Lissovoi et~al\mbox{.}(2019)]%
        {LissovoiOW19}
\bibfield{author}{\bibinfo{person}{Andrei Lissovoi}, \bibinfo{person}{Pietro~S.
  Oliveto}, {and} \bibinfo{person}{John~Alasdair Warwicker}.}
  \bibinfo{year}{2019}\natexlab{}.
\newblock \showarticletitle{On the time complexity of algorithm selection
  hyper-heuristics for multimodal optimisation}. In
  \bibinfo{booktitle}{\emph{Conference on Artificial Intelligence, {AAAI}
  2019}}. \bibinfo{publisher}{{AAAI} Press}, \bibinfo{pages}{2322--2329}.
\newblock


\bibitem[Lissovoi et~al\mbox{.}(2020a)]%
        {LissovoiOW20aaai}
\bibfield{author}{\bibinfo{person}{Andrei Lissovoi}, \bibinfo{person}{Pietro~S.
  Oliveto}, {and} \bibinfo{person}{John~Alasdair Warwicker}.}
  \bibinfo{year}{2020}\natexlab{a}.
\newblock \showarticletitle{How the duration of the learning period affects the
  performance of random gradient selection hyper-heuristics}. In
  \bibinfo{booktitle}{\emph{Conference on Artificial Intelligence, {AAAI}
  2020}}. \bibinfo{publisher}{{AAAI} Press}, \bibinfo{pages}{2376--2383}.
\newblock


\bibitem[Lissovoi et~al\mbox{.}(2020b)]%
        {LissovoiOW20ecj}
\bibfield{author}{\bibinfo{person}{Andrei Lissovoi}, \bibinfo{person}{Pietro~S.
  Oliveto}, {and} \bibinfo{person}{John~Alasdair Warwicker}.}
  \bibinfo{year}{2020}\natexlab{b}.
\newblock \showarticletitle{Simple hyper-heuristics control the neighbourhood
  size of randomised local search optimally for {LeadingOnes}}.
\newblock \bibinfo{journal}{\emph{Evolutionary Computation}}
  \bibinfo{volume}{28} (\bibinfo{year}{2020}), \bibinfo{pages}{437--461}.
\newblock


\bibitem[Lissovoi et~al\mbox{.}(2023)]%
        {LissovoiOW23}
\bibfield{author}{\bibinfo{person}{Andrei Lissovoi}, \bibinfo{person}{Pietro~S.
  Oliveto}, {and} \bibinfo{person}{John~Alasdair Warwicker}.}
  \bibinfo{year}{2023}\natexlab{}.
\newblock \showarticletitle{When move acceptance selection hyper-heuristics
  outperform {M}etropolis and elitist evolutionary algorithms and when not}.
\newblock \bibinfo{journal}{\emph{Artificial Intelligence}}
  \bibinfo{volume}{314} (\bibinfo{year}{2023}), \bibinfo{pages}{103804}.
\newblock


\bibitem[Neumann et~al\mbox{.}(2022)]%
        {NeumannSW22}
\bibfield{author}{\bibinfo{person}{Frank Neumann}, \bibinfo{person}{Dirk
  Sudholt}, {and} \bibinfo{person}{Carsten Witt}.}
  \bibinfo{year}{2022}\natexlab{}.
\newblock \showarticletitle{The compact genetic algorithm struggles on {C}liff
  functions}. In \bibinfo{booktitle}{\emph{Genetic and Evolutionary Computation
  Conference, GECCO 2022}}. \bibinfo{publisher}{{ACM}},
  \bibinfo{pages}{1426--1433}.
\newblock


\bibitem[Neumann and Wegener(2007)]%
        {NeumannW07}
\bibfield{author}{\bibinfo{person}{Frank Neumann} {and} \bibinfo{person}{Ingo
  Wegener}.} \bibinfo{year}{2007}\natexlab{}.
\newblock \showarticletitle{Randomized local search, evolutionary algorithms,
  and the minimum spanning tree problem}.
\newblock \bibinfo{journal}{\emph{Theoretical Computer Science}}
  \bibinfo{volume}{378} (\bibinfo{year}{2007}), \bibinfo{pages}{32--40}.
\newblock


\bibitem[Neumann and Witt(2010)]%
        {NeumannW10}
\bibfield{author}{\bibinfo{person}{Frank Neumann} {and}
  \bibinfo{person}{Carsten Witt}.} \bibinfo{year}{2010}\natexlab{}.
\newblock \bibinfo{booktitle}{\emph{Bioinspired Computation in Combinatorial
  Optimization -- Algorithms and Their Computational Complexity}}.
\newblock \bibinfo{publisher}{Springer}.
\newblock


\bibitem[Paix{\~{a}}o et~al\mbox{.}(2017)]%
        {PaixaoHST17}
\bibfield{author}{\bibinfo{person}{Tiago Paix{\~{a}}o},
  \bibinfo{person}{Jorge~P{\'{e}}rez Heredia}, \bibinfo{person}{Dirk Sudholt},
  {and} \bibinfo{person}{Barbora Trubenov{\'{a}}}.}
  \bibinfo{year}{2017}\natexlab{}.
\newblock \showarticletitle{Towards a runtime comparison of natural and
  artificial evolution}.
\newblock \bibinfo{journal}{\emph{Algorithmica}}  \bibinfo{volume}{78}
  (\bibinfo{year}{2017}), \bibinfo{pages}{681--713}.
\newblock


\bibitem[Rajabi and Witt(2021)]%
        {RajabiW21gecco}
\bibfield{author}{\bibinfo{person}{Amirhossein Rajabi} {and}
  \bibinfo{person}{Carsten Witt}.} \bibinfo{year}{2021}\natexlab{}.
\newblock \showarticletitle{Stagnation detection in highly multimodal fitness
  landscapes}. In \bibinfo{booktitle}{\emph{Genetic and Evolutionary
  Computation Conference, GECCO 2021}}. \bibinfo{publisher}{{ACM}},
  \bibinfo{pages}{1178--1186}.
\newblock


\bibitem[Rajabi and Witt(2022)]%
        {RajabiW22}
\bibfield{author}{\bibinfo{person}{Amirhossein Rajabi} {and}
  \bibinfo{person}{Carsten Witt}.} \bibinfo{year}{2022}\natexlab{}.
\newblock \showarticletitle{Self-adjusting evolutionary algorithms for
  multimodal optimization}.
\newblock \bibinfo{journal}{\emph{Algorithmica}}  \bibinfo{volume}{84}
  (\bibinfo{year}{2022}), \bibinfo{pages}{1694--1723}.
\newblock


\bibitem[Rajabi and Witt(2023)]%
        {RajabiW23}
\bibfield{author}{\bibinfo{person}{Amirhossein Rajabi} {and}
  \bibinfo{person}{Carsten Witt}.} \bibinfo{year}{2023}\natexlab{}.
\newblock \showarticletitle{Stagnation detection with randomized local search}.
\newblock \bibinfo{journal}{\emph{Evolutionary Computation}}
  \bibinfo{volume}{31} (\bibinfo{year}{2023}), \bibinfo{pages}{1--29}.
\newblock


\bibitem[Rowe and Aishwaryaprajna(2019)]%
        {RoweA19}
\bibfield{author}{\bibinfo{person}{Jonathan~E. Rowe} {and}
  \bibinfo{person}{Aishwaryaprajna}.} \bibinfo{year}{2019}\natexlab{}.
\newblock \showarticletitle{The benefits and limitations of voting mechanisms
  in evolutionary optimisation}. In \bibinfo{booktitle}{\emph{Foundations of
  Genetic Algorithms, {FOGA} 2019}}. \bibinfo{publisher}{{ACM}},
  \bibinfo{pages}{34--42}.
\newblock


\bibitem[Wald(1944)]%
        {Wald44}
\bibfield{author}{\bibinfo{person}{Abraham Wald}.}
  \bibinfo{year}{1944}\natexlab{}.
\newblock \showarticletitle{On cumulative sums of random variables}.
\newblock \bibinfo{journal}{\emph{Annals of Mathematical Statistics}}
  \bibinfo{volume}{15} (\bibinfo{year}{1944}), \bibinfo{pages}{283--296}.
\newblock


\bibitem[Whitley et~al\mbox{.}(2018)]%
        {WhitleyVHM18}
\bibfield{author}{\bibinfo{person}{Darrell Whitley}, \bibinfo{person}{Swetha
  Varadarajan}, \bibinfo{person}{Rachel Hirsch}, {and} \bibinfo{person}{Anirban
  Mukhopadhyay}.} \bibinfo{year}{2018}\natexlab{}.
\newblock \showarticletitle{Exploration and exploitation without mutation:
  solving the jump function in ${\Theta(n)}$ time}. In
  \bibinfo{booktitle}{\emph{Parallel Problem Solving from Nature, {PPSN} 2018,
  Part {II}}}. \bibinfo{publisher}{Springer}, \bibinfo{pages}{55--66}.
\newblock


\bibitem[Witt(2023)]%
        {Witt23}
\bibfield{author}{\bibinfo{person}{Carsten Witt}.}
  \bibinfo{year}{2023}\natexlab{}.
\newblock \showarticletitle{How majority-vote crossover and
  estimation-of-distribution algorithms cope with fitness valleys}.
\newblock \bibinfo{journal}{\emph{Theoretical Computer Science}}
  \bibinfo{volume}{940} (\bibinfo{year}{2023}), \bibinfo{pages}{18--42}.
\newblock


\end{thebibliography}

\ifthenelse{\boolean{arxiv}}
{
\newpage
\appendix
\section*{Appendix}

This appendix contains material omitted in the conference version~\cite{DoerrDLS23} for reasons of space. 

\subsection*{Detailed Proof of Theorem \ref{thm:combination}}


We split the random waiting time $T$ in two random waiting times,
$$T = T_1 + T_2,$$
where $T_1$ is the expected time to reach the local maximum from the $0$-bit and $T_2$ is the expected time to reach the $1$-bit from the local maximum. 

We will use the same type of proof as in Section \ref{sec:upper_bound}. Again we define a new distance function, for an $\alpha$ that we will determine later,
$$d(x) = 
\begin{cases}
    n - m - \Vert x \Vert_1, &\text{if } \Vert x \Vert_1 \leq n-m; \\
    \alpha(\Vert x \Vert_1 - (n-m)), & \text{otherwise.}
  \end{cases}
$$

We first prove the following lemma. 

\begin{lemma} \label{lemma:LowerBoundOnExpectation}
For $\alpha > 1$ and $p=\frac{m}{2 \alpha en}$, if we denote by $(X^{(t)})_{t\geq0}$ the sequence of states of a run of the \MAHH algorithm on $\jump_m$, then while $n-m- \Vert x \Vert_1 \geq 0$, i.e., while we are on the left of the local optimum,
$$\mathbb{E}\left[d\left(X^{(t)}\right)-d\left(X^{(t+1)}\right)\mid d\left(X^{(t)}\right)\right] \geq \frac{d\left(X^{(t)}\right)}{en}. $$
\end{lemma}

\begin{proof}
We begin by introducing notation for the two events describing the two directions in which our joint heuristic can progress. We use $C_k(X)$ to denote the event in which we start in state $X$ and move $k$ bits closer to the local optimum, while staying left of the local optimum,  and $C(X)$ the event in which we start at state $X$ and get strictly closer to the optimum, while staying on the left of the local optimum. We furthermore use $A^{\text{left}}_k(X)$ for the event in which we start in state $X$ and move $k$ bits away on the left of the local optimum and $A^{\text{left}}$ for the event to drift strictly away from the local optimum on the left. We also use $A^{\text{right}}_k(X)$ for the event in which we start in state $X$ and move $k$ bits away on the right of the local optimum and $A^{\text{right}}$ for the event to drift strictly away from the local optimum on the right. Additionally, we use $A_k\left(X\right)$ to denote the event in which we start in state $X$ and move to a position, which is $k$ bits away from the local optimum, and $A(X)$ to denote the random variable of the number of bits we move away from the local optimum when starting in state $X$. Finally, we define $S(X)$ as the event in which we start in state $X$ and stay in the same state.

\begin{align*}
\mathbb{E}\left[d\left(X^{(t+1)} \right)\right.&\mid\left. d\left(X^{(t)}\right)\right]  \\
=& \sum_{k=1}^{d(X^{(t)})} \left(d\left(X^{(t)}\right)-k\right) \mathbb{P}\left(C_k\left(X^{(t)}\right)\right) \\
&+\sum_{k=1}^{m} \alpha k \mathbb{P}\left(A^{\text{right}}_k\left(X^{(t)}\right)\right) \\ 
&+ \sum_{k=1}^{n-m-d(X^{(t)})} \left(d\left(X^{(t)}\right)+k\right) \mathbb{P}\left(A^{\text{left}}_k\left(X^{(t)}\right)\right) \\ 
&+ d\left(X^{(t)}\right) \mathbb{P}\left(S\left(X^{(t)}\right)\right)\\
\leq& d\left(X^{(t)}\right) -\!\!\!\! \sum_{k=1}^{d(X^{(t)})} \!\!\!\!k \mathbb{P}\left(C_k\left(X^{(t)}\right)\right) + 2\alpha \sum_{k=1}^n k \mathbb{P}\left(A_k\left(X^{(t)}\right)\right) \\
\leq& d\left(X^{(t)}\right) - \mathbb{P}(C(X)) + 2\alpha \mathbb{E}[A(X)].
\end{align*}

We furthermore use $W$ to denote the event in which we flip only one 0-bit to a 1-bit. We have that
\begin{eqnarray*}
\mathbb{P}(C(X)) &\geq& \mathbb{P}(W) \\
&\geq& \frac{1}{n}\left(1-\frac{1}{n}\right)^{n-1}\left(d\left(X^{(t)}\right)+m\right) \\
&\sim& \frac{d\left(X^{(t)}\right)+m}{2en}, 
\end{eqnarray*}
for $n$ large enough.
Moreover, we denote by $N(X)$ the random number of bits flipped, 
\begin{eqnarray*}
\mathbb{E}[A(X)] &\leq& p\mathbb{E}[N(X)]\\
&\leq& p,
\end{eqnarray*}
as the distance away from the local optimum is always inferior to the number of bits flipped.

Therefore, for $p=\frac{m}{2 \alpha en}$, we obtain
$$\mathbb{E}\left[d\left(X^{(t)}\right)-d\left(X^{(t+1)}\right)\mid d\left(X^{(t)}\right)\right] \geq \frac{d\left(X^{(t)}\right)}{en}. $$
\end{proof}

We can now bound $T_1$. Indeed, we have, 
$$T_1 = \min \{ t :  d\left(X^{(t+1)}\right) = 0 \}.$$
By the multiplicative drift theorem (Theorem \ref{thm:mult-drift}) applied with $s_0 = d(X^{(0)}) = n-m$, $s_{min}=1$ and $\delta = \frac{1}{en}$, we obtain,
$$\mathbb{E}\left[T_1\right] = O(n \log n).$$

\begin{lemma}
\label{lemma:drift3}
For $p=\frac{m}{8en}$, there exists $\Delta > 0$ such that if we denote by $(X^{(t)})_{t\geq0}$ the sequence of states of a run of the \MAHH algorithm on $\jump_m$, then while $n-m- \Vert x \Vert_1 \leq 0$, i.e., while we are on the right of the local optimum,
$$\mathbb{E}\left[d\left(X^{(t)}\right)-d\left(X^{(t+1)}\right)\mid d\left(X^{(t)}\right)\right] \geq \frac{d\left(X^{(t)}\right)\Delta}{n}. $$
\end{lemma}

\begin{proof}
We again make use of the  events and random variables defined in the proof of Lemma \ref{lemma:LowerBoundOnExpectation}. Moreover, use $D(X)$ to denote the event in which we start in state $X$ and move closer to the local optimum, while staying on its right. We furthermore, separate $A(X)$ the drift away from the optimum into two random variables : $A^{\text{left}}(X)$ and $A^{\text{right}}(X)$. 
For the same reason as the case on the left of the local optimum, we have,

\begin{align*}
\mathbb{E}\left[d\left(X^{(t+1)}\right)\right.&\left.\mid d\left(X^{(t)}\right)\right]  \\
\leq& d\left(X^{(t)}\right) - \alpha \mathbb{P}(D(X)) + \mathbb{E}[A(X)] \\
 \leq& d\left(X^{(t)}\right) - \alpha\frac{\Vert x\Vert_1}{n}\left(1-\frac{1}{n}\right)^{n-1}+ \mathbb{E}[A(X)]\\
 \leq& d\left(X^{(t)}\right) - \alpha\frac{n-m}{ne}+ \mathbb{E}\left[A^{\text{right}}(X)\right]+\mathbb{E}\left[A^{\text{left}}(X)\right]\\
 \leq& d\left(X^{(t)}\right) - \frac{\alpha}{2e}+p\alpha+\mathbb{E}\left[A^{\text{left}}(X)\right].
\end{align*}
And if we further denote the event in which we move from state $X^{(t)}$ to $n-m-i$ by $F_{d\left(X^{(t)}\right)-i}$, as well as the event in which we flip $i$ 1-bits to 0s by $F$,
\begin{eqnarray*}
\mathbb{E}[A^{\text{left}}(X)] &\leq& \sum_{i=d\left(X^{(t)}\right)+1}^{+\infty}i\mathbb{P}\left(F_{d\left(X^{(t)}\right)-i}\right) \\
&\leq& \sum_{i=d\left(X^{(t)}\right)+1}^{+\infty}i\mathbb{P}\left(F\right) \\
&\leq& \sum_{i=d\left(X^{(t)}\right)+1}^{+\infty}i\frac{1}{n^i}\left(1-\frac{1}{n}\right)^{n-i}\binom{\Vert x\Vert_1}{i}\\
&\leq& \sum_{i=\alpha}^{+\infty}\frac{1}{(i-1)!}.
\end{eqnarray*}

This is the rest of a converging sum, and therefore tends towards 0 when $\alpha \to \infty$. Therefore, for $\alpha$ big enough, $\mathbb{E}\left[A^{\text{left}}(X)\right]\leq\frac{\alpha}{2e}.$ In this case $\alpha=4$ works. This choice of $\alpha$ results in the choice $p=\frac{m}{8en}$. Furthermore, $\Delta = \frac{\alpha}{2e} -\mathbb{E}\left[A^{\text{left}}(X)\right] =  \frac{2}{e} -\mathbb{E}\left[A^{\text{left}}(X)\right]. $

\begin{eqnarray*}
\mathbb{E}\left[d\left(X^{(t+1)}\right)\mid d\left(X^{(t)}\right)\right]  &\leq& d\left(X^{(t)}\right) - \frac{\alpha}{2e}+p\alpha+\mathbb{E}\left[A^{\text{left}}(X)\right] \\
&\leq& d\left(X^{(t)}\right) - \Delta \\
&\leq& d\left(X^{(t)}\right)\left(1 - \frac{\Delta}{n}\right). 
\end{eqnarray*}

 By taking $\Delta \leq \frac{1}{e}$, we obtain a unified drift along the runnning of the algorithm
 $$\mathbb{E}\left[d\left(X^{(t)}\right)-d\left(X^{(t+1)}\right)\mid d\left(X^{(t)}\right)\right] \geq \frac{d\left(X^{(t)}\right)\Delta}{n}. $$

\end{proof}

To bound $T_2$, we define phases of random random times $(P_i)_{i \geq 0}$ when $X$ returns to $\{x^{loc},x^*\}$,
\begin{align*}
        P_0 &= T_1; \\
        P_{i+1} &= \min\{k \text{ : } k \geq P_{i}+1 \text{ and } X_{k} \in \{x^{\text{loc}},x^*\} \}.    
\end{align*}

Therefore, we have
$$T_2 = \sum_{i = 1}^{N} P_{i}-P_{i-1}, $$ 
where $N$ is the random number of phases until $X$ reaches $x^*$. As the $(X^{(t)})_{t\geq0}$  verifies the Markov property, the $\left(P_{i+1} - P_{i}\right)_{i \geq 0}$ are independent and following the same law.

We first prove the following lemma,
\begin{lemma}
\label{lm:drift4}
If $p \leq \frac{m}{8e n}$, then we have along the running of the \MAHH algorithm on $\jump_m$,
$$\mathbb{E}\left[P_{i+1}-P_i\right] = O(m).$$
\end{lemma}

\begin{proof}
In order to prove this lemma, we will use the multiplicative drift again. 

If the random walk stays on the local optimum, the expected duration of the phase is $1$.

In the second case, we can use Lemma \ref{lemma:drift3} with initial position $s_0 = d(X^{(P_i+1)})$. In this case, the multiplicative drift theorem gives an expected duration of $O(n\log s_0)$.

Therefore , we have, 
\begin{eqnarray*}
\mathbb{E}\left[P_{i+1}-P_i\right] &\leq & 1 + pO\left(\mathbb{E} \left[n \log s_0 \right]\right) \\
&\leq& O\left(m\mathbb{E} \left[\log s_0 \right]\right).
\end{eqnarray*}

The expected number of bits flipped at each step is equal to one, i.e., $\mathbb{E} \left[N(X)\right] = 1$. The expected value $\mathbb{E} \left[s_0 \right]$ is the expected number of bits at which we land from the local optimum after one run of mutations, therefore it is also a constant. Thanks to the concavity of $\log$, we know that $\mathbb{E} \left[\log s_0 \right]$ is smaller than a constant that does not depend on $n$. Therefore, 
$$ \mathbb{E}\left[P_{i+1}-P_i\right] =O(m).$$

\end{proof}

The path starting at the local optimum and going directly to the $1$-bit with the mutation operator by only flipping one $0$-bit has a probability of $m(1 - \frac{1}{n})^{n-m}\frac{1}{n^m}$.
Therefore the probability that the phase ends in $x^*$ is 
\begin{eqnarray*}
\mathbb{P}\left[X_{P_{i+1}} = x^*\right] &\geq& \left(1 - \frac{1}{n}\right)^{n-m}\frac{1}{n^m} \\
&\geq& \frac{1}{4n^m}.
\end{eqnarray*}

\begin{lemma}
\label{lm:drift5}
We have for a run of the \MAHH algorithm on $\jump_m$,
    $$\mathbb{E}\left[N\right] = O\left( \min\left(en^m, \frac{8^{m-1}(en)^{2m-1}}{m!m^{m-1}}\right)\right).$$
\end{lemma}

\begin{proof}
At each new phase $P_i$, for $i \geq 0$, the algorithm has the same probability of reaching the global maximum during the next phase via a global mutation of exactly the right bits with probability at least $ \frac{1}{en^m}$. 
We can furthermore consider the option of arriving at the global optimum as the result of a succession of individual bit flips moving us in the right direction, as in Lemma \ref{lm:N}, the probability of taking such a path is at least $\frac{m!}{(en)^m}p^{m-1}$.

Therefore, in the geometric distribution of the random variable $N$ we can take both success probabilities into account in the distributional parameter, which can be lower bounded as follows, 
$\mathbb{P}\left[X_{P_{i+1}} = x^*\right] \geq \frac{1}{en^m} + \frac{m!}{(en)^m}p^{m-1}\geq \max\left(\frac{1}{en^m}, \frac{m!}{(en)^m}p^{m-1}\right)$, for $p=\frac{m}{8en},$
$$\mathbb{E}\left[N\right] = \frac{1}{\mathbb{P}\left[X_{P_{i+1}} = x^*\right]} \leq \min\left(en^m, \frac{8^{m-1}(en)^{2m-1}}{m!m^{m-1}}\right).$$



\end{proof}

In Lemma \ref{lm:drift5} we observe how the combination of algorithms allows us to increase the probability of reaching the global optimum, which serves as motivation to further study the combination of well chosen algorithms in future research. 

We can now combine the results of Lemmas~\ref{lm:drift4} and \ref{lm:drift5} with Wald's formula (Theorem \ref{thm:wald}), to obtain that if $p = \frac{m}{8en}$
\begin{eqnarray*}
\mathbb{E}\left[T_2\right] &=& \mathbb{E}\left[N\right]\mathbb{E}\left[P_1-P_0\right]\\
&=&O(m\mathbb{E}\left[N\right])\\
&=& O\left(m \min\left(en^m, \frac{8^{m-1}(en)^{2m-1}}{m!m^{m-1}}\right)\right).
\end{eqnarray*}

}{}
  
	}
\end{document}